\documentclass{article}
\pdfoutput=1
\usepackage[margin=1in]{geometry}
\usepackage{natbib}
\usepackage{diagbox}
\usepackage{enumerate}
\usepackage{authblk}
\usepackage[utf8]{inputenc} 
\usepackage[T1]{fontenc}    
\usepackage{hyperref}       
\usepackage{url}            
\usepackage{booktabs}       
\usepackage{amsfonts}       
\usepackage{nicefrac}       
\usepackage{microtype}      
\usepackage{algorithmic}
\usepackage[utf8]{inputenc} 
\usepackage[T1]{fontenc}    
\usepackage{hyperref}       
\usepackage{url}            
\usepackage{booktabs}       
\usepackage{amsfonts}       
\usepackage{nicefrac}       
\usepackage{microtype}      
\usepackage{graphicx}
\usepackage[table]{xcolor}
\usepackage{textcomp}
\usepackage{gensymb}
\usepackage{subcaption}
\usepackage{verbatim}
\usepackage{mathrsfs}
\usepackage{amsmath}
\usepackage{amsthm}
\usepackage{wrapfig}
\usepackage{lipsum}
\usepackage{xcite}
\usepackage{xr}
\usepackage{adjustbox}
\usepackage[normalem]{ulem}
\usepackage[ruled,longend]{algorithm2e}
\makeatletter
\def\BState{\State\hskip-\ALG@thistlm}
\makeatother

\newtheorem{theorem}{Theorem}
\newtheorem{Proposition}[theorem]{Proposition}

\newtheorem{lemma}[theorem]{Lemma}
\newtheorem{remark}[theorem]{Remark}
\newtheorem{cor}[theorem]{Corollary}
\newtheorem{proposition}[theorem]{Proposition}
\newtheorem{fact}[theorem]{Fact}

\usepackage[normalem]{ulem}
\usepackage[ruled,longend]{algorithm2e}
\makeatletter
\def\BState{\State\hskip-\ALG@thistlm}
\makeatother

\makeatletter
\newcommand*{\addFileDependency}[1]{
  \typeout{(#1)}
  \@addtofilelist{#1}
  \IfFileExists{#1}{}{\typeout{No file #1.}}
}
\makeatother

\newcommand{\bE}{\mathbb{E}}

\newcommand{\norm}[1]{\left\lVert#1\right\rVert}
\newcommand{\tr}[1]{\text{trace}(#1)}
\newenvironment{sproof}{%
  \proof}{\endproof}

\usepackage{stackengine}

\makeatletter
\newcommand{\distas}[1]{\mathbin{\overset{#1}{\kern\z@\sim}}}%

\newcommand{\beqs}{\vspace{0mm}\begin{eqnarray}}
\newcommand{\eeqs}{\vspace{0mm}\end{eqnarray}}
\newcommand{\barr}{\begin{array}}
\newcommand{\earr}{\end{array}}

\newcommand{\alphav}{\boldsymbol{\alpha}}

\newcommand{\supp}{Appendix\xspace}

\newcommand{\rd}{\color{red}}
\newcommand{\bk}{\color{black}}

\title{A Unified Framework for Tuning Hyperparameters in Clustering Problems}

\author[1]{\small Xinjie Fan}
\author[1]{\small Yuguang Yue}
\author[1]{\small Purnamrita Sarkar}
\author[2]{\small Y. X. Rachel Wang}

\affil[1]{\footnotesize Department of Statistics and Data Science, University of Texas at Austin}
\affil[2]{\footnotesize School of Mathematics and Statistics, University of Sydney}
\affil[ ]{\textit {xfan@utexas.edu, yuguang@utexas.edu, purna.sarkar@austin.utexas.edu, rachel.wang@sydney.edu.au}}

\begin{document}

\maketitle

\begin{abstract}

Selecting hyperparameters for unsupervised learning problems is challenging in general due to the lack of ground truth for validation. Despite the prevalence of this issue in statistics and machine learning, especially in clustering problems, there are not many methods for tuning these hyperparameters with theoretical guarantees. In this paper, we provide a framework with provable guarantees for selecting hyperparameters in a number of distinct models. We consider both the subgaussian mixture model and network models to serve as examples of i.i.d. and non-i.i.d. data. We demonstrate that the same framework can be used to choose the Lagrange multipliers of penalty terms in semidefinite programming (SDP) relaxations for community detection, and the bandwidth parameter for constructing kernel similarity matrices for spectral clustering. By incorporating a cross-validation procedure, we show the framework can also do consistent model selection for network models. Using a variety of simulated and real data examples, we show that our framework outperforms other widely used tuning procedures in a broad range of parameter settings. 

\end{abstract}

\section{Introduction}
A standard statistical model has parameters, which characterize the underlying data distribution;  an inference algorithm to learn these parameters typically involve hyperparameters  (or tuning parameters). 
Popular examples include the penalty parameter in regularized regression models, the number of clusters in clustering analysis, the bandwidth parameter in kernel based clustering, nonparameteric density estimation or regression methods (\cite{wasserman2006all,tibshirani2015statistical}), to name but a few. It is well-known that selecting these hyperparameters may require repeated training to search through different combinations of plausible hyperparameter values and often has to rely on good heuristics and domain knowledge from the user. 

A classical method to do automated hyperparameter tuning is the nonparametric procedure Cross Validation (CV) (\cite{stone1974cross, zhang1993model}) which has been used extensively in machine learning and statistics (\cite{hastie2005elements}).
 CV has been studied extensively in supervised learning settings, particularly in low dimensional linear models (\cite{shao1993linear, yang2007consistency}) and penalized regression in high dimension (\cite{wasserman2009high}). Other notable stability based methods for model selection in similar supervised settings include \cite{breiman1996heuristics, bach2008bolasso, meinshausen2010stability, lim2016estimation}. Finally, a large number of empirical methods exist in the machine learning literature for tuning hyperparameters in various training algorithms (\cite{bergstra2012random, bengio2000gradient, snoek2012practical, bergstra2011algorithms}), most of which do not provide theoretical guarantees. 


In contrast to the supervised setting with i.i.d. data used in many of the above methods, in this paper, we consider \textit{unsupervised} clustering problems with possible dependence structure in the datapoints. We propose an overarching framework for hyperparameter tuning and model selection for a variety of probabilistic clustering models. Here the challenge is two-fold. Since labels are not available, choosing a criterion for evaluation and in general a method for selecting hyperparameters is not easy. One may consider splitting the data in different folds and selecting the model or hyperparameter with the most stable solution. However, for multiple splits of the data, the inference algorithm may get stuck at the same local optima, and thus stability alone can lead to a suboptimal solution (\cite{von2010clustering}). In~\cite{wang2010consistent, fang2012selection}, the authors overcome this by redefining the number of clusters as one that gives the most stable clustering for a given algorithm. In~\cite{meila2018tell}, a semi-definite program (SDP) maximizing an inner product criterion is performed for each clustering solution, and the value of the objective function is used to evaluate the stability of the clustering. The analysis is done without any model assumptions. The second difficulty arises if there is dependence structure in the datapoints, which necessitates careful splitting procedures in a CV-based procedure.

To illustrate the generality of our framework, we focus on subgaussian mixtures and the statistical network models like the Stochastic Blockmodel (SBM) and the Mixed Membership Stochastic Blockmodel (MMSB)  as two representative models for i.i.d. data and non i.i.d. data, where clustering is a natural problem. We propose a unified framework with provable guarantees to do  hyperparameter tuning and model selection in these models. More specifically, our contributions can be summarized as below:

\noindent 1. Our framework can provably tune the following \textbf{hyperparameters}:
    \begin{enumerate}[(a)]
        \item Lagrange multiplier of the penalty term in a type of semidefinite relaxation for community detection problems in SBM;
        \item Bandwidth parameter used in kernel spectral clustering for subgaussian mixture models.
    \end{enumerate} 
\noindent 2. We have consistent \textbf{model selection}, i.e. determining number of clusters:
    \begin{enumerate}[(a)]
        \item When the model selection problem is embedded in the choice of the Lagrange multiplier in another type of SDP relaxation for community detection in SBM;
        
        \item General model selection for the Mixed Membership Stochastic Blockmodel (MMSB), which includes the SBM as a sub-model.  
    \end{enumerate}

We choose to focus on model selection for network-structured data, because  there already is an extensive repertoire of empirical and provable methods including the gap statistic~\citep{tibs2001gap}, silhouette index~\citep{ROUSSEEUW198753}, the slope criterion~\citep{Birge2001}, eigen-gap~\cite{von2007tutorial}, penalized maximum likelihood~\citep{leroux1992}, information theoretic approaches (AIC~\citep{Bozdogan1987ModelSA}, BIC~\citep{keribin2000,drtonjrssb}, minimum message length~\citep{figueiredo2002mml}), spectral clustering and diffusion based methods~\citep{Maggioni2018LearningBU,little2017spec}  for i.i.d mixture models. We discuss the related work on the other models in the following subsection.
\subsection{Related Work}
\label{subsec:lit_review}
\textbf{Hyperparameters and model selection in network models:} In network analysis, 
while a number of methods exist for selecting the true number of communities (denoted by $r$) with consistency guarantees including \cite{lei2016goodness,wang2017,le2015estimating,bickel2016hypothesis} for SBM, and \cite{fan2019simple} and \cite{han2019universal} for more general models such as the degree-corrected mixed membership blockmodel,
these methods have not been generalized to other hyperparameter selection problems. For CV-based methods, existing strategies involve node splitting (\cite{chen2018network}), or edge splitting (\cite{li2016network}). In the former, it is established that CV prevents underfitting for model selection in SBM. In the latter, a similar one-sided consistency result for Random Dot Product Models (RDPG) (\cite{young2007random}, which includes SBM as a special case) is shown. This method has also been empirically applied to tune other hyperparameters, though no provable guarantee was provided. 


In terms of algorithms for community detection or clustering, SDP methods have gained a lot of attention
 (\cite{abbe2015exact,amini2018semidefinite,Guedon2016, cai2015robust, hajek2016achieving}) due to their strong theoretical guarantees. Typically, SDP based methods can be divided into two broad categories. The first one maximizes a penalized trace of the product of the adjacency matrix and an unnormalized clustering matrix (see definition in Section~\ref{subsec:setup}). Here the hyperparameter is the Lagrange multiplier of the penalty term~\cite{amini2018semidefinite,cai2015robust,chen2018network,Guedon2016}. In this formulation, the optimization problem does not need to know the number of clusters. However, it is implicitly required in the final step which obtains the memberships from the clustering matrix. 
 
 The other class of SDP methods uses a trace criterion with a normalized clustering matrix (definition in Section~\ref{subsec:setup})~\citep{Peng:2007,Yan2019CovariateRC,mixon2017sdp}. Here the constraints directly use the number of clusters.~\citep{yan2017provable} use a penalized alternative of this SDP to do provable model selection for SBMs. However, most of these methods require appropriate tuning of the Lagrange multipliers, which are themselves hyperparameters. Usually the theoretical upper and lower bounds on these hyperparameters involve unknown model parameters, which are nontrivial to estimate. The proposed method in~\cite{abbe2015recovering} is agnostic of model parameters, but it involves a highly-tuned and hard to implement spectral clustering step (also noted by \cite{perry2017semidefinite}).
 
 In this paper, we  use a SDP from the first class (SDP-1) to demonstrate our provable tuning procedure, and another SDP from the second class (SDP-2) to establish consistency guarantee for our model selection method.

\textbf{Spectral clustering with mixture model:} 
In statistical machine learning literature, analysis of spectral clustering typically is done in terms of the Laplacian matrix built from an appropriately constructed similarity matrix of the datapoints. There has been much work~\citep{hein2005,hein2006uniform,vonLuxburg2007,belkin2003laplacian,gine2006empirical} on establishing different forms of asymptotic convergence of  the Laplacian. Recently~\cite{lffler2019optimality} have established error bounds for spectral clustering that uses the gram matrix as the similarity matrix. In~\cite{srivastava2019robust} error bounds are obtained for a variant of spectral clustering for the Gaussian kernel in presence of outliers. Most of the existing tuning procedures for the bandwidth parameter of the Gaussian kernel are heuristic and do not have provable guarantees. Notable methods include~\cite{vonLuxburg2007}, who choose an analogous parameter, namely the radius $\epsilon$ in an $\epsilon$-neighborhood graph
 ``as the length of the
longest edge in a minimal spanning tree of the fully connected
graph on the data points.'' Other discussions on selecting the bandwidth can be found in~\citep{hein2005,coifman2008random} and~\citep{schiebinger2015}. \cite{shi2008data} propose a data dependent way to set the bandwidth parameter by suitably normalizing the $95\%$ \text{quantile of } a vector containing $5\%$ quantiles of distances from each point.  


We now present our problem setup in Section~\ref{sec:notation}. Section~\ref{sec:knownk} proposes and analyzes our hyperparameter tuning method MATR for networks and subgaussian mixtures. Next, in Section~\ref{sec:uknownk}, we present MATR-CV and the related consistency guarantees for model selection for SBM and MMSB models. Finally, Section~\ref{sec:exp} contains detailed simulated and real data experiments and we conclude with paper with a discussion in Section~\ref{sec:discussion}.
\bk

 \section{Preliminaries and Notations}
 \vspace{-.5em}
 \label{sec:notation}
 \subsection{Notations}
 \vspace{-.5em}
Let $(C_1, ..., C_{r})$ denote a partition of $n$ data points into $r$ clusters; $m_i = |C_i|$ denote the size of $C_i$. Denote $\pi_\text{min} = \min_i m_i/n$. The cluster membership of each node is represented by a $n\times r$ matrix $Z$, with $Z_{ij} = 1$ if data point $i$ belongs to cluster $j$, and $0$ otherwise. Since $r$ is the true number of clusters, $Z^TZ$ is full rank. Given $Z$, the corresponding unnormalized clustering matrix is $ZZ^T$, and the normalized clustering matrix is $Z (Z^TZ)^{-1}Z^T$. $X$ can be either a normalized or unnormalized clustering matrix, and will be made clear. We use $\tilde{X}$ to denote the matrix returned by SDP algorithms, which may not be a clustering matrix. Denote $\mathcal{X}_{r}$ as the set of all possible normalized clustering matrices with cluster number $r$. Let $Z_0$ and $X_0$ be the membership and normalized clustering matrix from the ground truth. $\lambda$ is a general hyperparameter; although with a slight abuse of notation, we also use $\lambda$ to denote the Lagrange multiplier in SDP methods. 
For any matrix $X\in \mathbb{R}^{n\times n}$, let $X_{C_k,C_{\ell}}$ be  a matrix such that $X_{C_k,C_{\ell}}(i,j)=X(i,j)$ if $i\in C_k,j\in C_\ell$, and $0$ otherwise. $E_n$ is the $n\times n$ all ones matrix. We write $\langle A, B\rangle = \text{trace}(A^TB).$ Standard notations of $o, O,o_P, O_P, \Theta, \Omega$ will be used. By ``with high probability'', we mean with probability tending to one. 
 
 

 
 \subsection{Problem setup and motivation}
 \label{subsec:setup}
 We consider a general clustering setting where the data $\mathcal{D}$ gives rise to a $n\times n$ observed similarity matrix $\hat{S}$, where $\hat{S}$ is symmetric. Denote $\mathscr{A}$ as a clustering algorithm which operates on the data $\mathcal{D}$ with a hyperparameter $\lambda$ and outputs a clustering result in the form of $\hat{Z}$ or $\hat{X}$. Here note that $\mathscr{A}$ may or may not perform clustering on $\hat{S}$, and $\mathscr{A}$, $\hat{Z}$ and $\hat{X}$ could all depend on $\lambda$.   In this paper we assume that $\hat{S}$ has the form $\hat{S}=S+R$,
 where $R$ is a matrix of arbitrary noise, and $S$ is the ``population similarity matrix''. As we consider different clustering models for network-structured data and iid mixture data, it will be made clear what $\hat{S}$ and $S$ are in each context.
 
 \textbf{Assortativity (weak and strong): }In some cases, we require weak assortativity on the similarity matrix $S$ defined as follows. Suppose for $i,j\in C_k$, $S_{ij}=a_{kk}$.
Define the minimal difference between diagonal term and off-diagonal terms in the same row cluster as 
\begin{align}\label{eq:pgap}
p_{\text{gap}} = \min_k \left(a_{kk} - \max_{\substack{i\in C_k,j\in C_\ell\\\ell \neq k}} S_{ij}\right).
\end{align}
Weak assortativity requires $p_{\text{gap}}> 0$. This condition is similar to weak assortativity defined for blockmodels (e.g.~\cite{amini2018semidefinite}). It is  mild compared to strong assortativity requiring $\min_k a_{kk} - \max_{\substack{i\in C_k,j\in C_\ell\\\ell \neq k}} S_{ij}> 0$.

 \textbf{Stochastic Blockmodel (SBM):} The SBM is a generative model of networks with community structure on $n$ nodes. By first partitioning the nodes into $r$ classes which leads to a membership matrix $Z$, the $n\times n$ binary adjacency matrix $A$ is sampled from probability matrix $P=Z_{i}BZ_j^T 1(i\neq j)$.
 where $Z_i$ and $Z_j$ are the $i^{th}$ and $j^{th}$ row of matrix $Z$, $B$ is the $r\times r$ block probability matrix. The aim is to estimate node memberships given $A$. 
 We assume the elements of $B$ have order $\Theta(\rho)$ with $\rho\to 0$ at some rate.

\textbf{Mixed Membership Stochastic Blockmodel (MMSB):}\label{subsubsec:mmsb} The SBM can be restrictive when it comes to modeling real world networks. As a result, various extensions have been proposed. 
The mixed membership stochastic blockmodel (MMSB,~\citep{airoldi2008mmsb}) relaxes the requirement on the membership vector $Z_i$ being binary and allows the entries to be in $[0,1]^{r}$, such that they sum up to 1 for all $i$. We will denote this soft membership matrix by $\Theta$. 

Under the MMSB model, the $n\times n$ adjacency matrix $A$ is sampled from the probability matrix $P$ with $ P_{ij} =\Theta_{i}B\Theta_j^T 1(i\neq j)$.
 We use an analogous definition for normalized clustering matrix: $X=\Theta(\Theta^T\Theta)^{-1}\Theta$. Note that this reduces to the usual normalized clustering matrix when $\Theta$ is a binary cluster membership matrix.

\textbf{Mixture of sub-gaussian random variables:} Let $Y=[Y_1, \dots, Y_n]^T$ be a $n\times d$ data matrix. We consider a setting where $Y_i$ are generated from a mixture model with $r$ clusters,
 \begin{equation}\label{eq:mog}
     Y_i = \mu_a + W_i, \quad \bE(W_i)=0, \quad Cov(W_i)=\sigma_a^2 I, \qquad a=1, \dots, r,
 \end{equation}
  where $W_i$'s are independent sub-gaussian vectors. 




\textbf{Trace criterion:}
Our framework is centered around the trace $\langle \hat{S}, X_\lambda\rangle$, where $X_\lambda$ is the normalized clustering matrix associated with hyperparameter $\lambda$. This criterion is often used in relaxations of the k-means objective~\citep{mixon2017sdp,Peng:2007,yan2017provable} in the context of SDP methods. The idea is that the criterion is large when datapoints within the same cluster are more similar. This criterion is also used by~\cite{meila2018tell} for evaluating stability of a clustering solution, where the author uses SDP to maximize this criterion for each clustering solution. Of course, this makes the implicit assumption that $\hat{S}$  (and $S$) is assortative, i.e. datapoints within the same cluster have high similarity based on $\hat{S}$. While this is reasonable for iid mixture models, not all community structures in network models are assortative if we use the adjacency matrix $A$ as  $\hat{S}$. If all the communities in a network are dis-assortative, then one can just use $-A$ as $\hat{S}$. However, for the SBM or MMSB models, one may have a mixture of assortative and dis-assortative structure.
In what follows, we begin our discussion of hyperparameter tuning and model selection for SBM by assuming weak assortativity, both for ease of demonstration and the fact that our algorithms of interest, SDP methods, operate on weakly assortative networks. For MMSB, which includes SBM as a sub-model, we show the same criterion still works without assortativity if we choose $\hat{S}$ to be $A^2$ with the diagonal removed.

\section{Hyperparameter tuning with known $r$}
\label{sec:knownk}
In this section, we consider tuning hyperparameters when the true number of clusters $r$ is known. First, we provide two simulation studies to motivate this section. The detailed parameter settings for generating the data can be found in the \supp Section~\ref{sec:expdetail}.

As mentioned in Section~\ref{subsec:lit_review}, SDP is an important class of methods for community detection in SBM, but its performance can depend on the choice of the Lagrange multiplier parameter. We first consider a SDP formulation \citep{li2018convex}, which has been widely used with slight variations in the literature \citep{amini2018semidefinite,perry2017semidefinite,Guedon2016,cai2015robust,chen2018network}, \bk
 \begin{equation}
 \tag{SDP-1}
 \begin{split}
 \max \quad &\text{trace}(AX) - \lambda \text{trace}(X E_n)\\
 \text{s.t.} \quad & X\succeq 0 , X\geq 0, X_{ii} = {1} \text{ for } 1\leq i \leq n,\\
 \end{split}
 \label{SDP:YD}
 \end{equation}
where $\lambda$ is a hyperparameter. Typically, one then performs spectral clustering (that is, $k$-means on the top $r$ eigenvectors) on the output of the SDP to get the clustering result. In Figure~\ref{fig:yd_motivation} (a), we generate an adjacency matrix from the probability matrix described in \supp Section~\ref{sec:expdetail} and use \ref{SDP:YD} with tuning parameter $\lambda$ from 0 to 1. The accuracy of the clustering result is measured by the normalized mutual information (NMI) and shown in Figure~\ref{fig:yd_motivation} (a). We can see that different $\lambda$ values lead to widely varying clustering performance.

 As a second example, we consider a four-component Gaussian mixture model generated as described in \supp Section~\ref{sec:expdetail}. We perform spectral clustering ($k$-means on the top $r$ eigenvectors) on the widely used Gaussian kernel matrix (denoted $K$) with bandwidth parameter $\theta$. Figure~\ref{fig:yd_motivation}(b) shows the clustering performance using NMI as $\theta$ varies, and the flat region of suboptimal $\theta$ corresponds to cases when the two adjacent clusters cannot be separated well. 

 \begin{figure}[htp!]
 	\begin{subfigure}{.5\textwidth}
 		\centering
 		\includegraphics[width=0.7\linewidth]{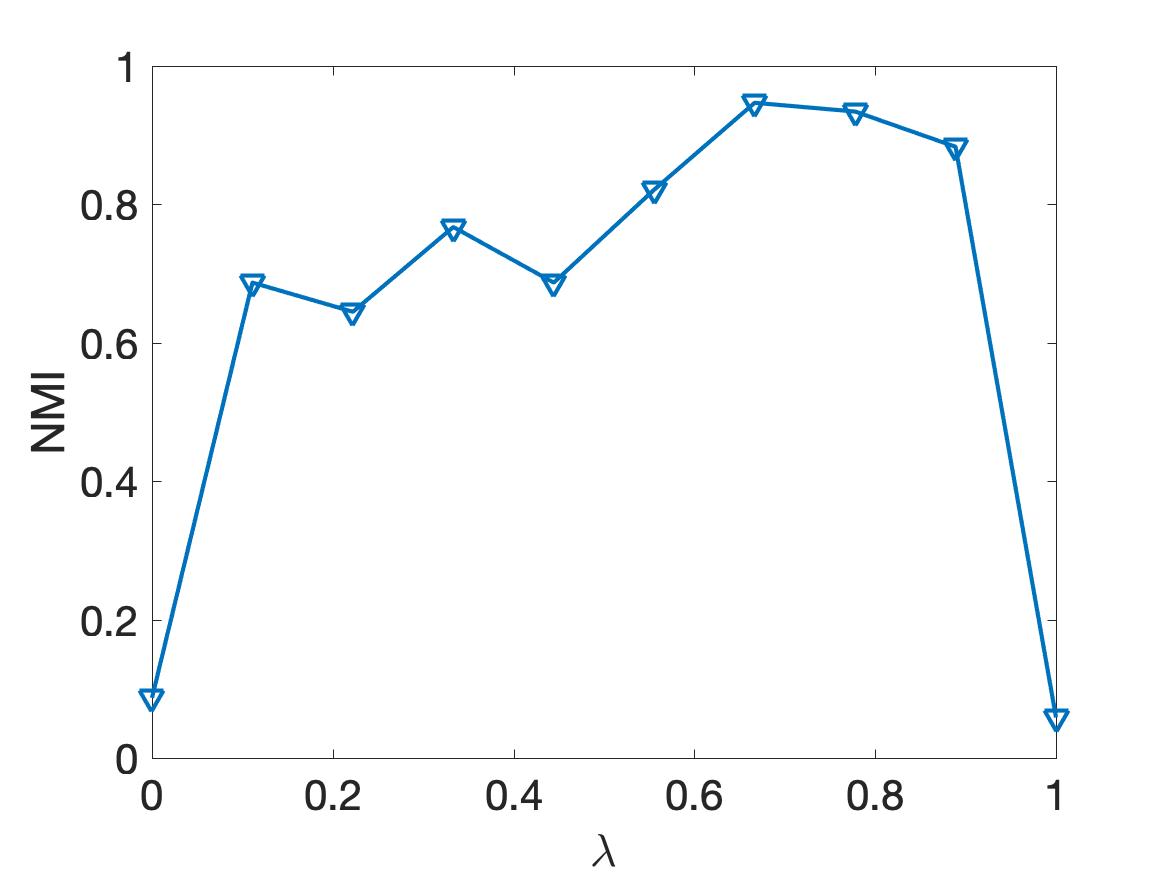}
 		\label{sfig_motivation4}
 		\caption{NMI v.s. $\lambda$}
 	\end{subfigure}
 	\begin{subfigure}{.5\textwidth}
 		\centering
 		\includegraphics[width=0.7\linewidth]{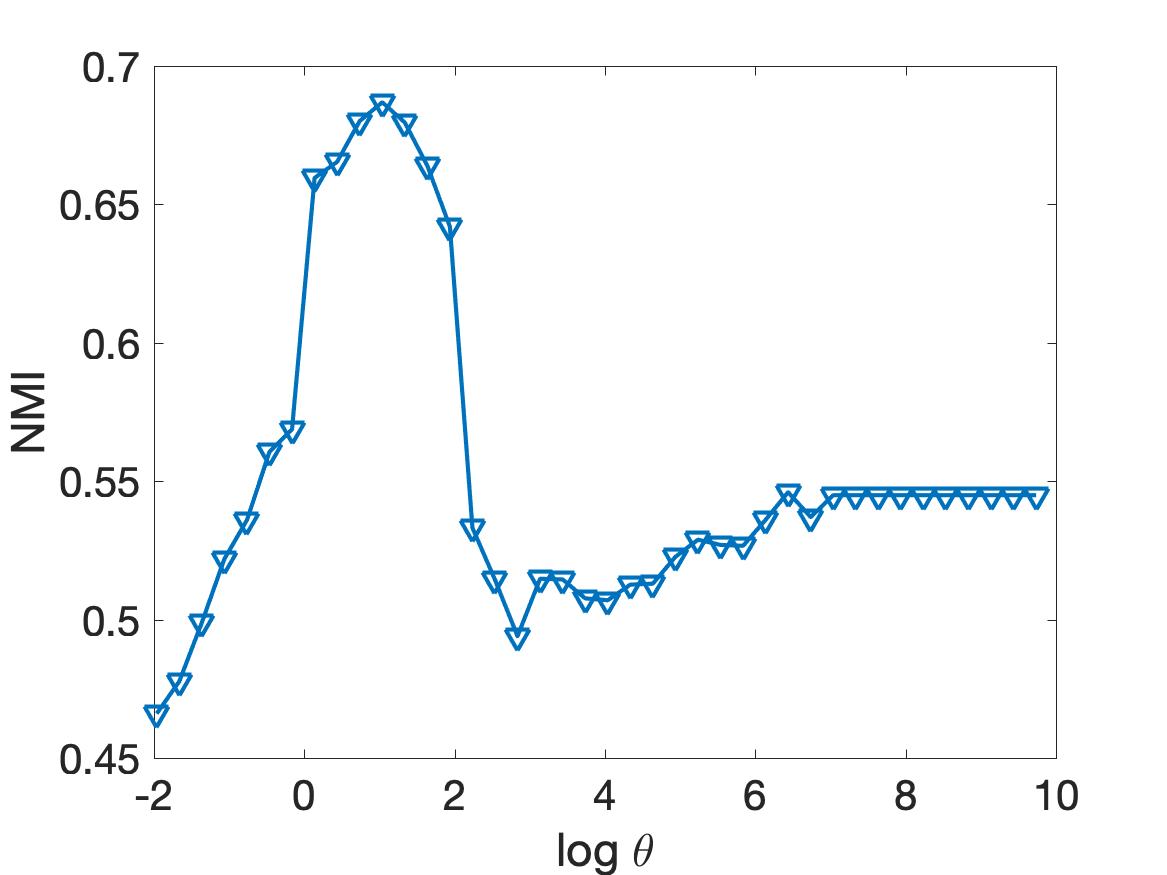} 
 		\label{sfig_motivation2}
 		\caption{NMI v.s. $\theta$}
 	\end{subfigure}
 	\caption{Tuning parameters in SDP and Spectral clustering; accuracy measured by normalized mutual information (NMI).}
 	\label{fig:yd_motivation}
 \end{figure}

We show that in the case where the true cluster number $r$ is known, an ideal hyperparameter $\lambda$ can be chosen by simply maximizing the trace criterion introduced in Section~\ref{subsec:setup}. The tuning algorithm (MATR) is presented in Algorithm~\ref{alg:kk}. It takes a general clustering algorithm $\mathscr{A}$, data $\mathcal{D}$ and similarity matrix $\hat{S}$ as inputs, and outputs a clustering result $\hat{Z}_{\lambda^*}$ with $\lambda^*$ chosen by maximizing the trace criterion. 

\begin{algorithm}\vspace{-1mm}
\textbf{Input:} clustering algorithm $\mathscr{A}$, data $\mathcal{D}$, similarity matrix $\hat{S}$, a set of candidates $\{\lambda_1, \cdots, \lambda_T\}$, number of clusters $r$\;
\textbf{Procedure:}

\For{$t = 1:T$}{run clustering on $\mathcal{D}$: $\hat{Z}_t = \mathscr{A}(\mathcal{D}, \lambda_t, r)$\;
compute normalized clustering matrix: $\hat{X}_t = \hat{Z}_t(\hat{Z}^T_t\hat{Z}_t)^{-1}\hat{Z}^T_t$\;
  compute inner product: $l_t = \langle \hat{S}, \hat{X}_t \rangle$\;
  }
$t^* = \text{argmax}(l_1, ..., l_T)$\;
\caption{MAx-TRace (MATR) based tuning algorithm for known number of clusters.}
\textbf{Output:} $\hat{Z}_{t^*}$
\label{alg:kk}
\end{algorithm}

We have the following theoretical guarantee for Algorithm~\ref{alg:kk}.
	
\begin{theorem}
	\label{thm:general_setup}
	Consider a clustering algorithm $\mathscr{A}$ with inputs $\mathcal{D}, \lambda, r$ and output $\hat{Z}_{\lambda}$.
	The similarity matrix $\hat{S}$ used for Algorithm \ref{alg:kk}(MATR) can be written as $\hat{S}=S+R$. We further assume $S$ is weakly assortative with $p_{\text{gap}}$ defined in Eq~\eqref{eq:pgap}, and $X_0$ is the normalized clustering matrix for the true binary membership matrix $Z_0$. Let $\pi_{\min}$ be the smallest cluster proportion, and $\tau := n \pi_\text{min} p_{\text{gap}}$.
    As long as there exists $\lambda_0\in \{\lambda_1, \dots, \lambda_T\}$, such that $\langle \hat{X}_{\lambda_0}, \hat{S}\rangle \geq \langle X_{0}, S\rangle - \epsilon$, Algorithm \ref{alg:kk} will output a $\hat{Z}_{\lambda^*}$, such that $$\norm{\hat{X}_{\lambda^*} - X_0}_F^2 \leq \frac{2}{\tau}(\epsilon+\sup_{X\in \mathcal{X}_{r}} |\langle X, R\rangle |),$$ where $\hat{X}_{\lambda^*}$ is the normalized clustering matrix associated with $\hat{Z}_{\lambda^*}.$
	
	
\end{theorem}

In other words, as long as the range of $\lambda$ we consider covers some optimal $\lambda$ value that leads to a sufficiently large trace criterion (compared with the true underlying $X_0$ and the population similarity matrix $S$), the theorem guarantees Algorithm~\ref{alg:kk} will lead to a normalized clustering matrix with small error. The deviation $\epsilon$ depends both on the noise matrix $R$ and how close the estimated $\hat{X}_{\lambda_0}$ is to the ground truth $X_0$, i.e. the performance of the algorithm. If both $\epsilon$ and $\sup_{X\in \mathcal{X}_r} |\langle X, R\rangle |$ are $o_P(\tau)$, then MATR will yield a clustering matrix which is weakly consistent. The proof is in the \supp Section~\ref{sec:proofformatr}. 

In the following subsections, we apply MATR to more specific settings, namely to select the Lagrange multiplier parameter in \ref{SDP:YD} for SBM and the bandwidth parameter in spectral clustering for sub-gaussian mixtures.

\subsection{Hyperparameter tuning for SBM}

We consider the problem of choosing $\lambda$ in~\ref{SDP:YD} for community detection in SBM. Here, the input to Algorithm \ref{alg:kk} -- the data $\mathcal{D}$ and similarity matrix $\hat{S}$ -- are both the adjacency matrix $A$. A natural choice of a weakly assortative $S$ is the conditional expectation of $A$, i.e. $P$ up to diagonal entries: let $\tilde{P}_{ij}=P_{ij}$ for $i\neq j$ and $\tilde{P}_{ii}=B_{kk}$ for $i\in C_k$. Note that $\tilde{P}$ is blockwise constant, and assortativity condition on $\tilde{P}$ translates naturally to the usual assortativity condition on $B$. As the output matrix $\tilde{X}$ from \ref{SDP:YD} may not necessarily be a clustering matrix, we use spectral clustering on $\tilde{X}$ to get the membership matrix $\hat{Z}$ required in Algorithm~\ref{alg:kk}. \ref{SDP:YD} together with spectral clustering is used as $\mathscr{A}$.

In Proposition~\ref{prop:sdp1recover} of the \supp, we show that \ref{SDP:YD} is strongly consistent, when applied to a general strongly assortative SBM with known $r$, as long as $\lambda$ satisfies: 
\begin{align}
    \max_{k\neq l}B_{k,l} + \Omega(\sqrt{\rho \log n /n\pi_\text{min}}) \leq \lambda \leq \min_k B_{kk}+ O( \sqrt{\rho \log n/n\pi_{\max}^2})
    \label{eq:lambda_range_sdp1}
\end{align}

An empirical way of choosing $\lambda$ was provided in \cite{cai2015robust}, which we will compare with in Section \ref{sec:exp}.
We first show a result complementary to Eq~\ref{eq:lambda_range_sdp1} under a SBM model with weakly assortative $B$, that for a specific region of $\lambda$, the normalized clustering matrix from \ref{SDP:YD} will merge two clusters with high probability. This highlights the importance of selecting an appropriate $\lambda$ since different values can lead to drastically different clustering result. The detailed statement and proof can be found in Proposition \ref{prop:sdp1merge} of the \supp Section~\ref{sec:merge1}.  

When we use Algorithm \ref{alg:kk} to tune $\lambda$ for $\mathscr{A}$\phantomsection \label{A_sdp1}, we have the following theoretical guarantee.

\begin{cor}\label{thm:fixk}
    Consider $A\sim SBM(B,Z_0)$ with weakly assortative $B$  and $r$ number of communities. Denote $\tau := n \pi_\text{min} \min_k (B_{kk}-\max_{\ell \neq k} B_{k\ell})$. If we have $\epsilon = o_P(\tau),\,
    r\sqrt{n\rho} = o(\tau), \, n\rho \geq c\log n$, 
    for some constant $c>0$, then as long as there exists $\lambda_0\in \{\lambda_1, \dots, \lambda_T\}$, such that $\langle \hat{X}_{\lambda_0}, A\rangle \geq \langle X_{0}, P\rangle - \epsilon$ , with $\mathscr{A}$ Algorithm \ref{alg:kk}(MATR) will output a $\hat{Z}_{\lambda^*}$, such that $\| \hat{X}_{\lambda^*} - X_0\|_F^2 = o_P(1),$ where $\hat{X}_{\lambda^*}$, $X_0$ are the normalized clustering matrices for $\hat{Z}_{\lambda^*}$, $Z_0$ respectively. 
\end{cor}

\begin{remark}
\begin{enumerate}
    \item Since $\lambda\in[0,1]$, to ensure the range of $\lambda$ considered overlaps with the optimal range in~\eqref{eq:lambda_range_sdp1}, it suffices to consider $\lambda$ choices from $[0,1]$. Then for  $\lambda$ satisfying Eq~\ref{eq:lambda_range_sdp1}, \ref{SDP:YD} produces $\tilde{X}=X_0$ w.h.p. if $B$ is strongly assortative. Since $\langle X_0, R \rangle = O_P(r\sqrt{n\rho})$, we can take $\epsilon = O(r\sqrt{n\rho})$, and the conditions in this corollary imply $\frac{r}{\sqrt{n\rho} \pi_{\text{min}}} \to 0$. Suppose all the communities are of comparable sizes, i.e. $\pi_{\text{min}}=\Theta(1/r)$, then the conditions only require $r=O(\sqrt{n})$ since $n\rho\to\infty$. 

    \item
    Since the proofs of Theorem~\ref{thm:general_setup} and Corollary~\ref{thm:fixk} are general, the conclusion is not limited to \ref{SDP:YD} and applies to more general community detection algorithms for SBM when $r$ is known. It is easy to see that a sufficient condition for the consistency of $\hat{X}_{\lambda^*}$ to hold is that there exists $\lambda_0$ in the range considered, such that $|\langle\hat{X}_{\lambda_0}-X_0, P\rangle|=o_P(\tau)$.
    
    \item
    We note that the specific application of Corollary~\ref{thm:fixk} to \ref{SDP:YD} leads to weak consistency of $\hat{X}_{\lambda^*}$ instead of strong consistency as originally proved for \ref{SDP:YD}. This is partly due to the generality of theorem (including the relaxation of strong assortativity on $B$ to weak assortativity) as discussed above, and the fact that we are estimating $\lambda$.  

\end{enumerate}
\end{remark}

\subsection{Hyperparameter tuning for mixtures of subgaussians}
\label{subsec:sc_subgaussian}
 In this case, the data $\mathcal{D}$ is $Y$ defined in Eq~\eqref{eq:mog}, the clustering algorithm $\mathscr{A}$ is spectral clustering (see motivating example in Section \ref{sec:knownk})
 on the Gaussian kernel  $K(i,j)=\exp\left(-\frac{\|Y_i-Y_j\|_2^2}{2\theta^2}\right)$. Note that one could use the similarity matrix as the kernel itself. However, this makes the trace criterion a function of the hyperparameter we are trying to tune, which compounds the difficulty of the problem. For simplicity, we use the negative squared distance matrix as $\hat{S}$, i.e. $\hat{S}_{ij}=-\|Y_i-Y_j\|_2^2$. The natural choice for $S$ would be the conditional expectation of $\hat{S}$ given the cluster memberships, which is blockwise constant, as in the case for SBM's. However, in this case, the convergence behavior is different from that of  blockmodels. In addition, this choice leads to a suboptimal error rate. Therefore we use a slightly corrected variant of the matrix as $S$ (also see~\citep{mixon2017sdp}), called the reference matrix:
 \begin{align}\label{eq:referenceSubg}
 S_{ij}=-\frac{d_{ab}^2}{2}-\max\left\{0,\frac{d_{ab}^2}{2}+2(W_i-W_j)^T(\mu_a-\mu_b)\right\} 1(i\in C_a,j\in C_b),
 \end{align}
 where $d_{ab}:=\|\mu_a-\mu_b\|$, $W_i$ is defined in Eq~\ref{eq:mog}. Note that for $i,j$ in the same cluster $S_{ij}=0$. Interestingly this reference matrix is random itself, which is a deviation from the $S$ used for network models.   
For MATR applied to select $\theta$, we have the following theoretical guarantee.

\begin{cor}\label{cor:mixture}
Let $\hat{S}$ be the negative squared distance matrix, and let $S$ be defined as in Eq~\ref{eq:referenceSubg}. Let $\delta_{\text{sep}}$ denote the minimum distance between cluster centers, i.e. $\min_{k\neq \ell} \|\mu_k-\mu_\ell\|$. Denote $\tau := n \pi_\text{min} \delta_{\text{sep}}^2/2$ and $\alpha=\pi_{\text{max}}/\pi_{\text{min}}$.
  As long as there exists $\theta_0\in \{\theta_1, \dots, \theta_T\}$, such that $\langle \hat{X}_{\theta_0}, \hat{S}\rangle \geq \langle X_{0}, S\rangle - n\pi_{\text{min}}\epsilon$ , Algorithm \ref{alg:kk}(MATR) will output a $\hat{Z}_{\theta^*}$, such that w.h.p.
  \begin{align*}
    \|\hat{X}_{\theta^*}-X_0\|_F^2&\leq C\frac{\epsilon+r\alpha\sigma_{\max}^2(\alpha+\min\{r,d\})}{\delta_{\text{sep}}^2}
\end{align*} where  $\sigma_{\max}$ is the largest operator norm of the covariance matrices of the mixture components,  $\hat{X}_{\theta^*}$ is the normalized clustering matrix for $\hat{Z}_{\theta^*}$ and $C$ is an universal constant.
\end{cor}

\begin{remark}
Note that, similar to SBMs, in this setting, $\epsilon$ has to be much smaller than $\delta^2_{\text{sep}}$ in order to guarantee small error. This will happen if the spectral clustering algorithm is supplied with an appropriate bandwidth parameter that leads to small error in estimating $X_0$ (see for example~\citep{srivastava2019robust}). This is satisfied by the condition $\theta_0\in\{\theta_1,\dots,\theta_T\}$ in Corollary~\ref{cor:mixture}.
\bk
\end{remark}
\section{Hyperparameter tuning with unknown $r$}
\label{sec:uknownk}

In this section, we adapt MATR to situations where the number of clusters is unknown to perform model selection. Similar to Section~\ref{sec:knownk}, we first explain the general tuning algorithm and state a general theorem to guarantee its performance. Then applications to specific models will be discussed in the following subsections. Since the applications we focus on are network models, we will present our algorithm with the data $\mathcal{D}$ being $A$ for clarity.
 
We show that MATR can be extended to model selection if we incorporate a cross-validation (CV) procedure. In Algorithm~\ref{alg:matr-cv}, we present the general MATR-CV algorithm which takes clustering algorithm $\mathscr{A}$, adjacency matrix $A$, and similarity matrix $\hat{S}$ as inputs. Compared with MATR, MATR-CV has two additional parts.

The first part (Algorithm~\ref{alg:nodesplit}) is to split nodes into two subsets for training and testing. This in turn partitions the adjacency matrix $A$ into four submatrices $A^{{11}}$, $A^{{22}}$, $A^{{21}}$ and its transpose, and similarly for $\hat{S}$. MATR-CV makes use of all the submatrices: $A^{{11}}$ for training, $A^{{22}}$ for testing, $A^{{11}}$ and $A^{{21}}$ for estimating the clustering result for nodes in $A^{{22}}$ as shown in Algorithm~\ref{alg:clustertest}, which is the second additional part. Algorithm~\ref{alg:clustertest} clusters testing nodes based on the training nodes cluster membership estimated from $A^{11}$, and the connections between training nodes and testing nodes $A^{21}$. 

\begin{minipage}{0.58\textwidth}
	\begin{algorithm}[H]
		\textbf{Input:} clustering algorithm $\mathscr{A}$, adjacency matrix $A$, similarity matrix $\hat{S}$, candidates $\{r_1, \cdots, r_T\}$, number of repetitions $J$, training ratio $\gamma_\text{train}$, trace gap $\Delta$\;
		\For{$j = 1:J$}
		{
			\For{$t = 1:T$}  
			{
				$\hat{S}^{11},\hat{S}^{21},\hat{S}^{22}$ $\leftarrow$ NodeSplitting($\hat{S}$, $n$, $\gamma_\text{train}$)\;
				$A^{11},A^{21},A^{22}$ $\leftarrow$ NodeSplitting($A$, $n$, $\gamma_\text{train}$)\;
				$\hat{Z}^{11} = \mathscr{A}(A^{11}, r_t)$\;
				$\hat{Z}^{22} = \text{ClusterTest}(A^{21},\hat{Z}^{11})$; 
				
				$\hat{X}^{22} = \hat{Z}^{22}(\hat{Z}^{22^T}\hat{Z}^{22})^{-1}\hat{Z}^{22^T}$\;
				$l_{r_t,j} = \langle \hat{S}^{22}, \hat{X}^{22} \rangle$\;
			}
			$r^*_j = \text{min}\{r_t: l_{r_t, j} \geq \max_t l_{r_t, j} - \Delta\}$\;
		}
		$\hat{r} = \text{median}\{r^*_j\}$
		
		\caption{MATR-CV.}
		\textbf{Output:} $\hat{r}$
		\label{alg:matr-cv}
	\end{algorithm}
\end{minipage}\hspace{5mm}
\begin{minipage}{0.35\textwidth}
	\begin{algorithm}[H]
		\textbf{Input:} $A$, $n$, $\gamma_{\text{train}}$\;
		Randomly split $[n]$ into $Q_1$, $Q_2$ of size $n\gamma_{\text{train}}$ and $n(1-\gamma_{\text{train}})$
		
		$A^{11}\leftarrow A_{Q_1, Q_1}$, $A^{21}\leftarrow A_{Q_2, Q_1}$, $A^{22} \leftarrow A_{Q_2, Q_2}$
		
		\textbf{Output:} $A^{11},A^{21},A^{22}$
		\caption{\text{NodeSplitting}}
		\label{alg:nodesplit} 
	\end{algorithm}
	
	\begin{algorithm}[H]
		\textbf{Input:} $A^{21}\in \{0,1\}^{n\times m}$, $\hat{Z}^{11}\in \{0,1\}^{m\times k}$\;
		$M\leftarrow A^{21}\hat{Z}^{11}({\hat{Z}^{11T}}\hat{Z}^{11})^{-1}$\;
		\For{$i = 1:n$}
		{
			$\hat{Z}^{22}(i,\arg\max M(i,:))=1$
		}
		\caption{$\text{ClusterTest}$}
		\textbf{Output:} $\hat{Z}^{22}$
		\label{alg:clustertest}
	\end{algorithm}

\end{minipage}

For each node in the testing set, using the estimated membership $\hat{Z}^{11}$, the corresponding row in $M$ counts the number of connections it has with nodes in the training set belonging to each cluster and normalizes the counts by the cluster sizes. Finally, the estimated membership $\hat{Z}^{22}$ is determined by a majority vote. For now we still assume $B$ is weakly assortative, so majority vote is reasonable. As we later extend to more general network structures in Section~\ref{subsec:mmsb}, we will also show how Algorithm~\ref{alg:clustertest} can be generalized. 

Like other CV procedures, we note that MATR-CV requires specifying a training ratio $\gamma_{\text{train}}$ and the number of repetitions $J$. Choosing any $\gamma_{\text{train}}=\Theta(1)$ does not affect our asymptotic results. Repetitions of splits are used empirically to enhance stability; theoretically we show asymptotic consistency for any random split. The general theoretical guarantee and the role of the trace gap $\Delta$ are given in the next theorem.

\begin{theorem}\label{thm:matrcv_general}
Given a candidate set of cluster numbers $\{r_1, \dots, r_T\}$ containing the true number of cluster $r$, let $\hat{X}^{22}_{r_t}$ be the normalized clustering matrix obtained from $r_t$ clusters, as described in MATR-CV. Assume the following is true:

\noindent(i) with probability at least $1-\delta_{under}$, 
$\max_{r_t<r}\langle \hat{S}^{22}, \hat{X}^{22}_{r_t}\rangle \leq \langle \hat{S}^{22}, X_0^{22}\rangle - \epsilon_\text{under};$

\noindent(ii) with probability at least $1-\delta_{over}$, 
$\max_{r<r_t\leq r_T}\langle \hat{S}^{22}, \hat{X}^{22}_{r_t}\rangle \leq \langle \hat{S}^{22}, X_0^{22}\rangle + \epsilon_\text{over};$

\noindent(iii) for the true $r$, with probability at least $1-\delta_{est}$, $\langle \hat{S}^{22}, \hat{X}^{22}_{r} \rangle \geq \langle \hat{S}^{22}, X_0^{22} \rangle -\epsilon_\text{est};$

\noindent(iv) there exists $\Delta>0$ such that $\epsilon_\text{est}+ \epsilon_\text{over} \leq \Delta < \epsilon_\text{under} - \epsilon_\text{est}.$

\noindent Here $\epsilon_\text{under}, \epsilon_\text{est}, \epsilon_\text{over} >0$. Then with probability at least $1-\delta_{under}-\delta_{over}-\delta_{est}$, MATR-CV will recover the true $r$ with trace gap $\Delta$.
\end{theorem}

The proof is deferred to the \supp Section~\ref{sec:proofformatrcv}. 

\begin{remark}
\begin{enumerate}

    \item 
    MATR-CV is also compatible with tuning multiple hyperparameters. For example, for \ref{SDP:YD}, if the number of clusters is unknown, then for each $\hat{r}$, we can run MATR to find the best $\lambda$ for the given $\hat{r}$, followed by running a second level MATR-CV to find the best $\hat{r}$. As long as the conditions in Theorems~\ref{thm:general_setup} and~\ref{thm:matrcv_general} are met, $\hat{r}$ and the clustering matrix returned will be consistent. 
    
    \item
    As will be seen in the applications below, the derivations of $\epsilon_{\text{under}}$ and $\epsilon_{\text{over}}$ are general and only depend on the properties of $\hat{S}$. On the other hand, $\epsilon_{\text{est}}$ measures the estimation error associated with the algorithm of interest and depends on its performance.
    
\end{enumerate}
 
\end{remark}

In what follows, we demonstrate MATR-CV can be applied to do model selection inherent to an SDP method for SBM and more general model selection for MMSB. While we still assume an assortative structure for the former model as required by the SDP method, the constraint is removed for MMSB. Furthermore, we use these two models to illustrate how MATR-CV works both when $\epsilon_{\text{est}}$ is zero (SBM) and nonzero (MMSB).

\subsection{Model selection for SBM}\label{sec:matrcv_sbm}

We consider the SDP algorithm introduced in \cite{Peng:2007, yan2017provable} as shown in \ref{SDP:BW-lambda} for community detection in SBM. Here $X$ is a normalized clustering matrix, and in the case of exact recovery $\tr{X}$ is equal to the number of clusters. In this way, $r$ is  implicitly chosen through $\lambda$, hence most of the existing model selection methods with consistency guarantees do not apply directly. \cite{yan2017provable} proposed to recover the clustering and $r$ simultaneously. However, $\lambda$ still needs to be empirically selected first. We provide a systematic way to do this. 
\begin{equation}
\tag{SDP-2-$\lambda$}
\max_X \quad \text{trace}(AX) - \lambda \text{trace}(X)\qquad \text{s.t.} \quad X\succeq 0 , X\geq 0, X\bf{1} = \bf{1}\\
\label{SDP:BW-lambda}
\end{equation}
We consider applying MATR-CV to an alternative form of \ref{SDP:BW-lambda} as shown in \ref{SDP:BW}, where the cluster number $r'$ appears explicitly in the constraint and is part of the input. \ref{SDP:BW} returns an estimated normalized clustering matrix, to which we apply spectral clustering to compute the cluster memberships. We name this algorithm $\mathscr{A_\text{SDP-2}}$. In this case, we use $A$ as $\hat{S}$, so $P$ is the population similarity matrix.  
\begin{equation}
\tag{SDP-2}
\max_X \quad \text{trace}(AX) \qquad \text{s.t.} \quad X\succeq 0 , X\geq 0,  \text{trace}(X) = r', X \bf{1} = \bf{1}\\
\label{SDP:BW}
\end{equation}

We have the following result ensuring MATR-CV returns a consistent cluster number.

\begin{theorem}
\label{cor:nok}
Suppose $A$ is generated from a SBM model with $r$ clusters and a weakly assortative $B$. We assume $r$ is fixed, and $\pi_{\min}\geq \delta >0$ for some constant $\delta$, and $n\rho/\log n \to \infty$. Given a candidate set of $\{r_1,\dots, r_T\}$ containing true cluster number $r$ and $r_T=\Theta(r)$, with high probability for $n$ large, MATR-CV returns the true number of clusters with
$\Delta=(1+B_{\max})\sqrt{r_{\max} \log n} + B_{\max}r_{\max}$,  where $r_{\max}:=\arg\max_{r_t}\langle A, \hat{X}_{r_t} \rangle$.
\end{theorem}

\begin{sproof}
We provide a sketch of the proof here, the details can be found in the \supp Section~\ref{sec:cor:nok}. We derive the three errors in Theorem~\ref{thm:matrcv_general}. In this case, we show that w.h.p., 
$
    \epsilon_{\text{under}} = \Omega(np_{\text{gap}}\pi_\text{min}^2/r^2) $, 
  $ \epsilon_{\text{over}} = (1+B_{\max})\sqrt{r_T\log n} + B_{\max}r,
$
and MATR-CV achieves exact recovery when given the true $r$, that is, $\epsilon_{\text{est}}=0$. Since $\epsilon_{\text{under}}\gg \epsilon_{\text{over}}$ under the conditions of the theorem, by Theorem~\ref{thm:matrcv_general}, taking $\Delta=\epsilon_{\text{over}}$ MATR-CV returns the correct $r$ w.h.p. Furthermore, we can remove the dependence of $\Delta$ on unknown $r$ by noting that $
    r_{\max}:=\arg\max_{r_t}\langle A, \hat{X}_{r_t} \rangle \geq r
    $ w.h.p., then it suffices to consider the candidate range $\{r_1, \dots, r_{\max}\}$. Thus $r_T$ and $r$ in $\Delta$ can be replaced with $r_{\max}$.
\end{sproof}

\begin{remark}
\begin{enumerate}
    \item
    Although we have assumed fixed $r$, it is easy to see from the order of $\epsilon_{\text{under}}$ and $\epsilon_{\text{over}}$ that the theorem holds for
$r^5/n \to0$, $r^{4.5}\sqrt{\log n} / (n\rho) \to 0$ if we let $\pi_{\min}=\Omega(1/r)$ for clarity. Many other existing works on SBM model selection assume fixed $r$. \cite{lei2016goodness} considered the regime  $r=o(n^{1/6})$. \cite{hu2017using} allowed $r$ to grow lineary up to a logarithmic factor, but at the cost of making $\rho$ fixed.  

    \item
    Asymptotically, $\Delta$ is equivalent to 
    $\Delta_{\text{SDP-2}}:=\sqrt{r_{\max}\log n}$.
    We will use $\Delta_{\text{SDP-2}}$ in practice when $r$ is fixed.

    \bk
\end{enumerate}
\end{remark}


\subsection{Model selection for MMSB}
\label{subsec:mmsb}
In this section, we consider model selection for the MMSB model as introduced in Section~\ref{subsec:setup} with a soft membership matrix $\Theta$, which is more general than the SBM model. As an example of estimation algorithm, we consider the SPACL algorithm proposed by \cite{Mao2017EstimatingMM}, which gives consistent parameter estimation when given the correct $r$. As mentioned in Section~\ref{subsec:setup}, a normalized clustering matrix in this case is defined analogously as $X=\Theta(\Theta^T\Theta)^{-1}\Theta^T$ for any $\Theta$. $X$ is still a projection matrix, and $X \mathbf{1}_n =\Theta (\Theta ^T \Theta)^{-1} \Theta ^T\mathbf{1}_n=\Theta (\Theta ^T \Theta)^{-1} \Theta ^T \Theta \mathbf{1}_r = \mathbf{1}_n, $ since $\Theta \mathbf{1}_r = \mathbf{1}_n$. Following \cite{Mao2017EstimatingMM}, we consider a Bayesian setting for $\Theta$: each row of $\Theta$, $\Theta_i \sim \text{Dirichlet} (\alphav), \alphav \in \mathbb{R}_{+}^{r}$. We assume $r$, $\alphav$ are all fixed constants. Note that the Bayesian setting here is only for convenience, and can be replaced with equivalent assumptions bounding the eigenvalues of $\Theta^T\Theta$. We also assume there is at least one pure node for
each of the $r$ communities for consistent estimation at the correct $r$.


MATR-CV can be applied to the MMSB model with a few modifications. (i) Replace all $\hat{Z}^{11}$ by $\hat{\Theta}^{11}$, the estimated soft memberships from the training graph. (ii) We take $\hat{S}=A^2-\text{diag}(A^2)$, $S=P^2-\text{diag}(P^2)$. This allows us to remove the assortativity requirement on $P$ and replace it with a full rank condition on $B$, which is commonly assumed in the MMSB literature. The fact that $P^2$ is always positive semi-definite will be used in the proof. The removal of $\text{diag}(A^2)$ and $\text{diag}(P^2)$ leads to better concentration, since $\text{diag}(A^2)$ is centered around a different mean. (iii) We change Algorithm \ref{alg:clustertest} to estimate $\hat{\Theta}^{22}$. Note that $P^{12}=\Theta^{11}B(\Theta^{22})^T$, thus we can view the estimation of $\Theta^{22}$ as a regression problem with plug-in estimators of $\Theta^{11}$ and $B$. In Algorithm~\ref{alg:clustertest}, we use an estimate of the form $\hat{\Theta}^{22} =A^{21} \hat{\Theta}^{11} ((\hat{\Theta}^{11})^T\hat{\Theta}^{11})^{-1} \hat{B}^{-1}$, where $\hat{B}$, $\hat{\Theta}^{11}$ are estimated from $A^{11}$. 

We have the following consistency guarantee for $\hat{r}$ returned by MATR-CV. 
\begin{theorem}
\label{thm:nok_mmsb}
Let $A$ be generated from a MMSB model (see Section~\ref{subsec:setup}) satisfying $\lambda^*(B)=\Omega(\rho)$, where $\lambda^*(B)$ is the smallest singular value of $B$. We assume $\sqrt{n\rho}/(\log n)^{1+\xi}\to\infty$ for some arbitrarily small $\xi>0$. Given a candidate set of $\{r_1, \dots, r_T\}$ containing $r$ and $r_T=\Theta(1)$, with high probability for large $n$, MATR-CV returns the true cluster number $r$ if $\Delta=O((n\rho)^{3/2}(\log n)^{1.01})$. 
\end{theorem}
\begin{sproof}
 We first show w.h.p., the underfitting and overfitting errors in Theorem~\ref{thm:matrcv_general} are
$
   \epsilon_{\text{under}} = \Omega(n^2\rho^2)$, $  
   \epsilon_{\text{over}} = O(n\rho\sqrt{\log n}).
$
To obtain $\epsilon_{\text{est}}$, we show that given the true cluster number, the convergence rate of the parameter estimates for the testing nodes obtained from the regression algorithm is the same as the convergence rate for the training nodes. This leads to $\epsilon_{\text{est}} = O((n\rho)^{3/2}(\log n)^{1+\xi})$. For convenience we pick $\xi=0.01$.  
 For details, see Section~\ref{subsec:proofmmsb} of the supplement. 
\end{sproof}

\begin{remark}
\begin{enumerate}
    \item  
    Compared with \cite{fan2019simple} and \cite{han2019universal}, which consider the more general degree-corrected MMSB model, our consistency result holds for $\rho\to0$ at a faster rate.   
    
    \item A practical note: due to the constant in the estimation error being tedious to determine, in this case we only know the asymptotic order of the gap $\Delta$. As has been observed in many other methods based on asymptotic properties (e.g. \cite{bickel2016hypothesis, lei2016goodness,  wang2017, hu2017using}), performing an adjustment for finite samples often improves the empirical performance. In practice we find that if the constant factor in $\Delta$ is too large, then we tend to underfit. To guard against this, we note that at the correct $r$, the trace difference   $\delta_{r, r-1}:=\langle\hat{S}, \hat{X}_{r}\rangle-\langle\hat{S}, \hat{X}_{r-1}\rangle$ should be much larger than $\Delta$. We start with $\Delta=(n\rho)^{3/2}(\log n)^{1.01}$ and find $\hat{r}$ by Algorithm~\ref{alg:matr-cv}; if $\delta_{\hat{r}, \hat{r}-1}$ is smaller than $\Delta$, we reduce $\Delta$ by half and repeat the step of finding $r^*_j$ in Algorithm~\ref{alg:matr-cv} until $\delta_{\hat{r}, \hat{r}-1} > \Delta$. This adjustment is much more computationally efficient than bootstrap corrections and works well empirically.

\end{enumerate}
\end{remark}

\section{Numerical experiments}
\label{sec:exp}
In this section, we present extensive numerical results on simulated and real data by applying MATR and MATR-CV to different settings considered in Sections~\ref{sec:knownk} and \ref{sec:uknownk}. 


\subsection{MATR on SBM with known number of clusters}\label{exp1}
We apply MATR to tune $\lambda$ in \ref{SDP:YD} for known $r$. Since $\lambda\in [0,1]$ for \ref{SDP:YD}, we choose $\lambda\in \{0,\cdots, 20\}/20$ in all the examples. 
For comparison we choose two existing data driven methods. The first method (CL,~\cite{cai2015robust}) sets $\lambda$ as the mean connectivity density in a subgraph determined by nodes with ``moderate'' degrees. 
The second is ECV (\cite{li2016network}) 
which uses CV with edge sampling to select the $\lambda$ giving the smallest loss on the test edges from a model estimated on training edges. We use a training ratio of 0.9 and the $L_2$ loss throughout. 

\noindent\textbf{Simulated data.}
Consider a strongly assortative SBM  as required by \ref{SDP:YD} for both equal sized and unequal sized clusters. The details of the experimental setting can be found in the \supp Section~\ref{sec:expdetail}. Standard deviations are calculated based on random runs of the each parameter setting. We present NMI comparisons for equal sized SBM ($n=400$, $r=4$) in Figure \ref{fig:tuning}(A), and unequal sized SBM (two with 100 nodes, and two with 50) in Figure \ref{fig:tuning}(B). In both, MATR outperforms others by a large margin as degree grows. 

\begin{figure}[ht]
	\begin{tabular}{@{\hspace{-2em}}c@{\hspace{-.6em}}c@{\hspace{-.6em}}c@{\hspace{-.6em}}c}
	\includegraphics[width=0.29\linewidth]{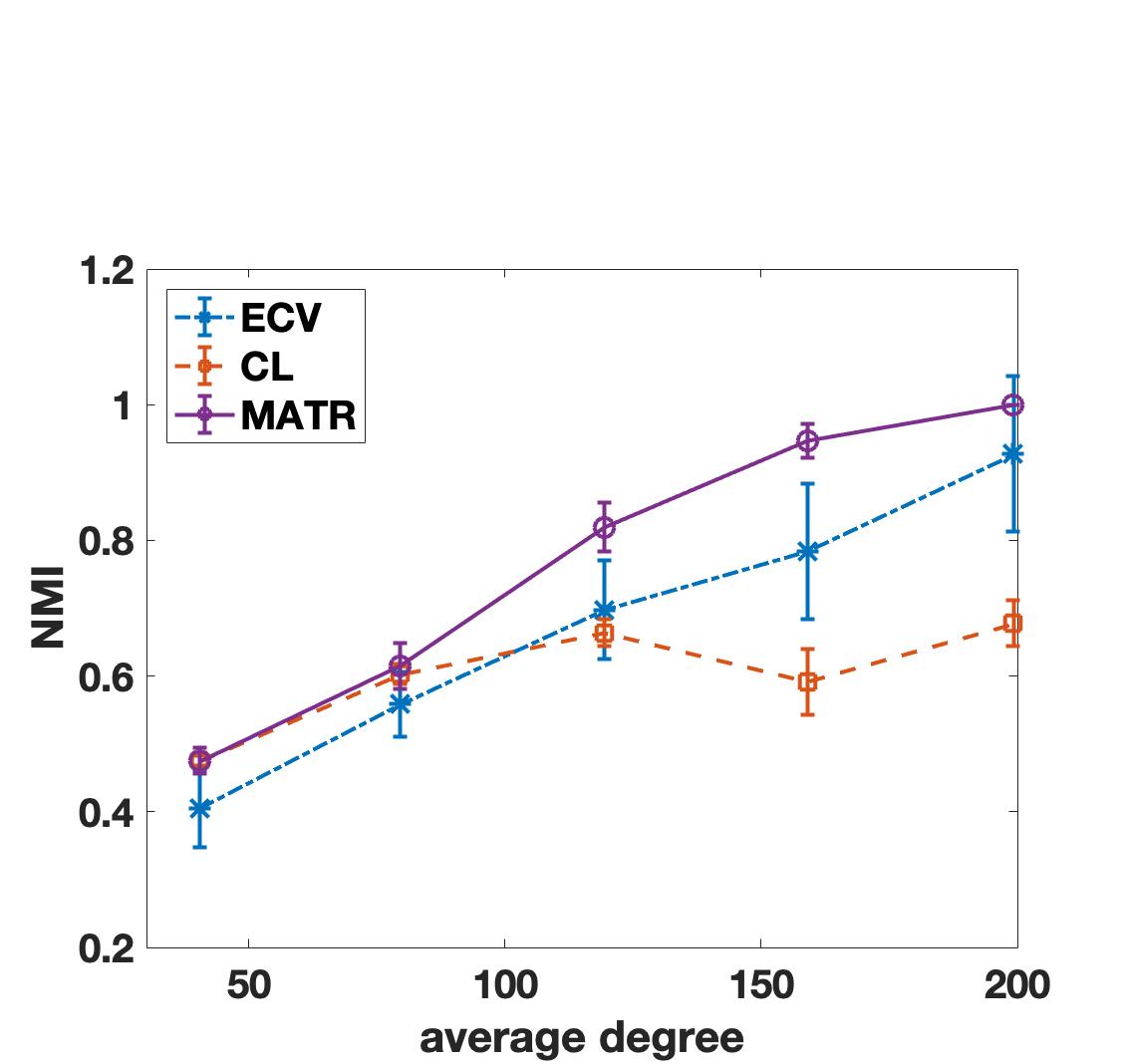} & \includegraphics[width=0.3\linewidth]{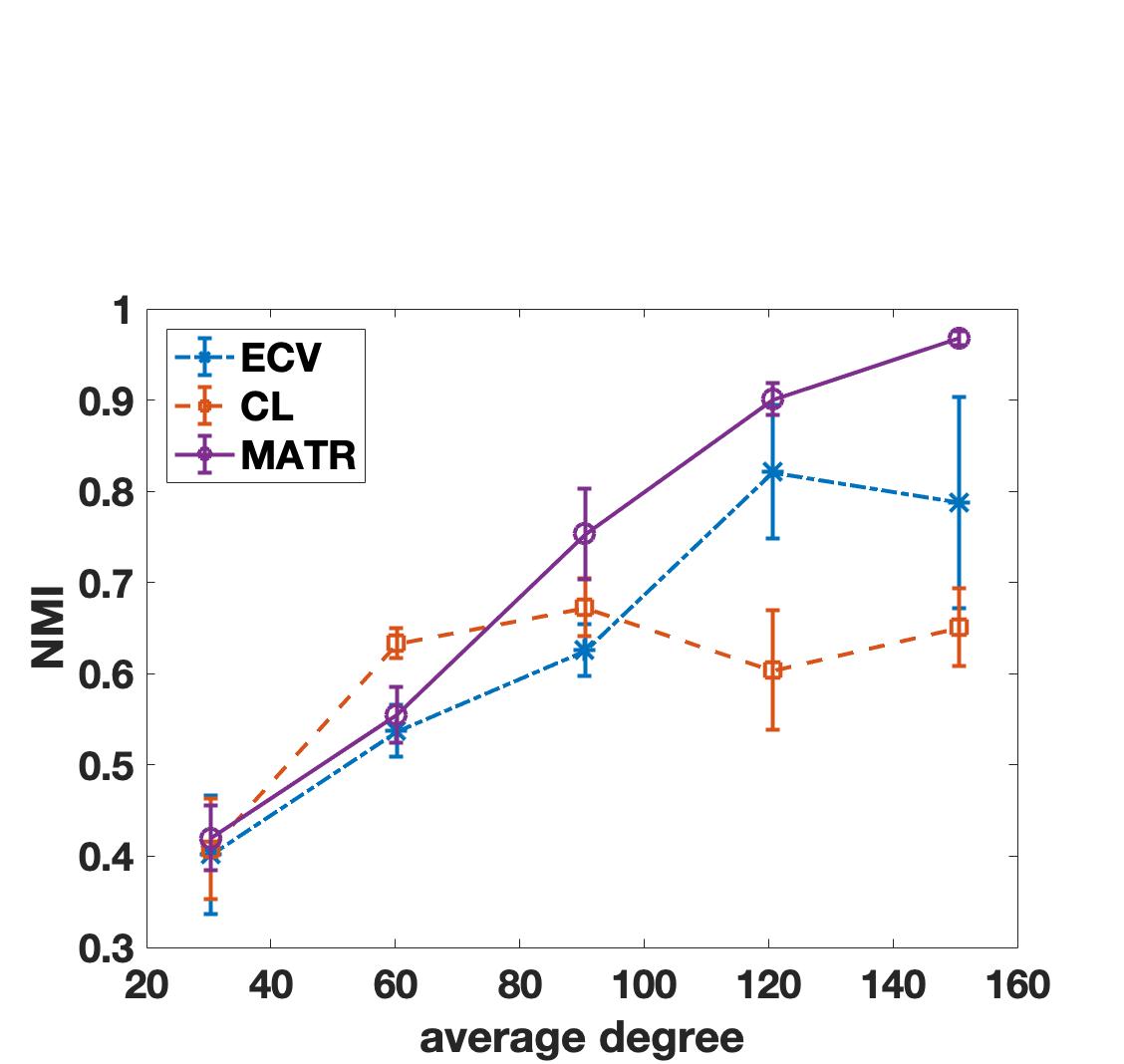}&
	\includegraphics[width=0.29\linewidth]{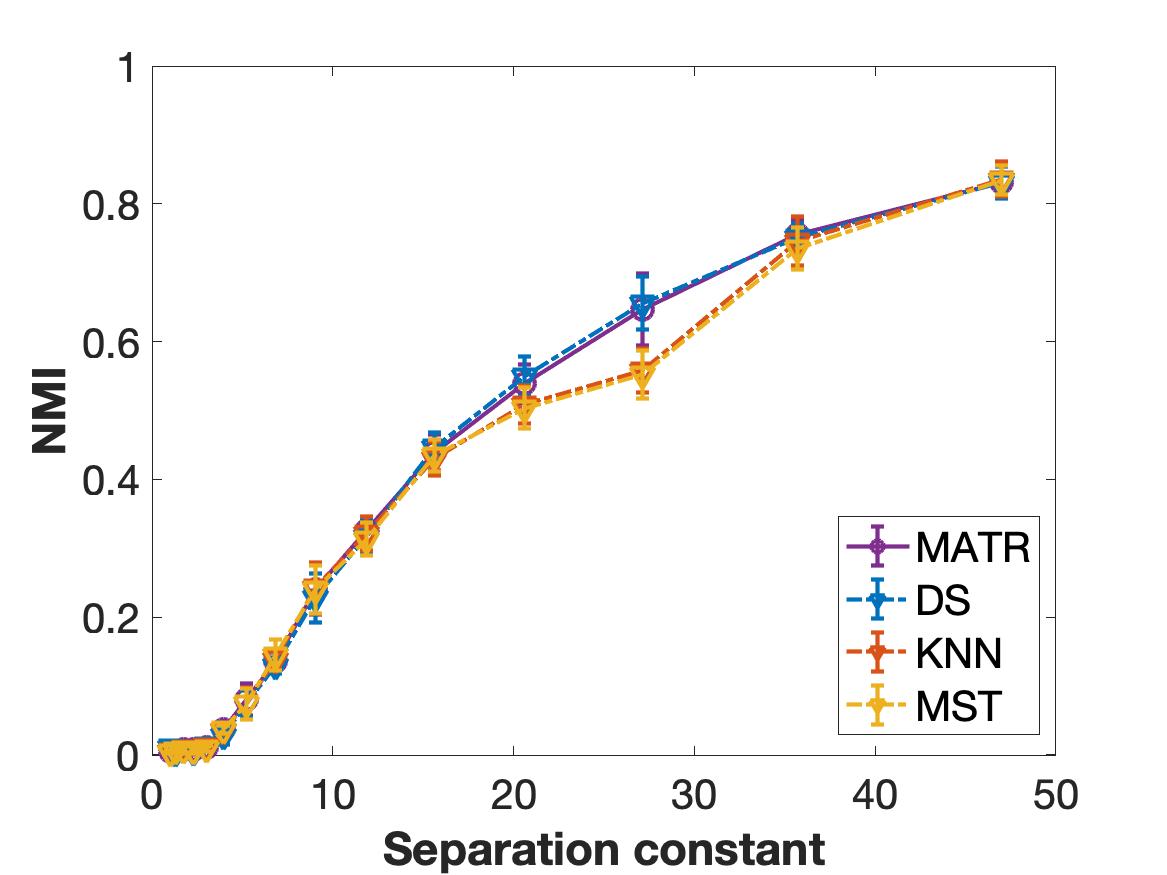} &
	\includegraphics[width=0.29\linewidth]{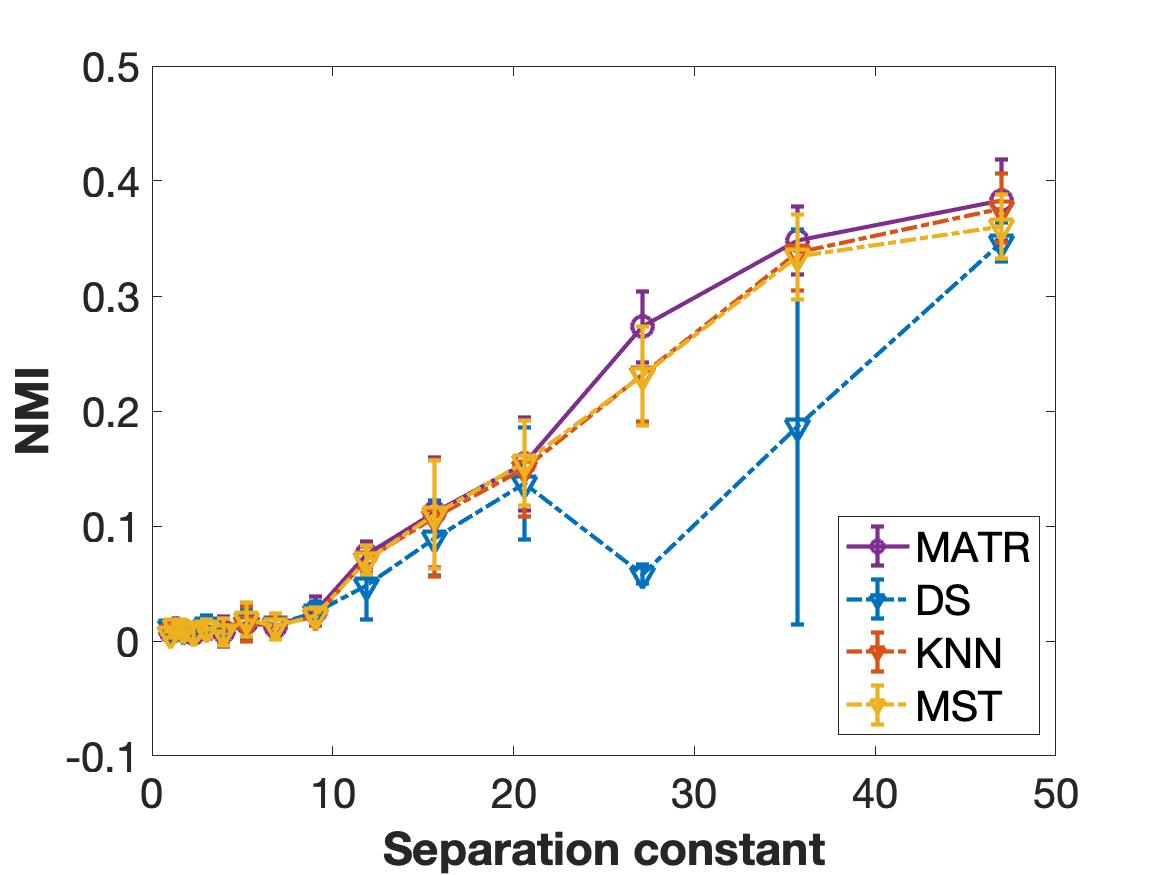}\\
	(A)&(B)&(C)&(D)
\end{tabular}
\caption{\label{fig:tuning}Comparison of NMI for tuning $\lambda$ for \ref{SDP:YD} for equal (A) and unequal sized (B) SBMs. Comparison of NMI for tuning bandwidth in spectral clustering for mixture models with (C) equal and (D) unequal mixing coefficients.}
\end{figure}

\noindent\textbf{Real data.}
We also compare MATR with ECV and CL on three real datasets: the football dataset \citep{girvan2002community}, the political books dataset and the political blogs dataset \citep{adamic2005political}. All of them are binary networks with $115,105$ and $1490$ nodes respectively. In the football dataset, the nodes represent teams and an edge is drawn between two teams if any regular-season games are played between them; there are $12$ clusters where each cluster represents the conference among teams, and games are more frequently between teams in the same conference. In the political blogs dataset, the nodes are weblogs and edges are hyperlinks between the blogs; it has $2$ clusters based on political inclination: "liberal" and "conservative". In the political books dataset, the nodes represent books and edges indicate co-purchasing on Amazon; the clusters represent $3$ categories based on manual labeling of the content: "liberal", "neutral" and "conservative".  The clustering performance of each method is evaluated by NMI and shown in Table \ref{T1}. MATR has performs the best out of the three methods  on the football dataset, and is tied with ECV on the political books dataset. MATR is not as good as CL on the poligical blogs dataset, but still outperforms ECV.  

 \begin{table}[h]
\begin{subtable}[t]{0.50\textwidth}
\centering
\begin{tabular}{|l|l|l|l|l|}
\hline
         & MATR   & ECV    & CL      \\
\hline
Football & 0.924 & 0.895 & 0.883 \\
\hline
Political blogs & 0.258 & 0.142 & 0.423\\
\hline
Political books & 0.549 & 0.549 & 0.525 \\
\hline
\end{tabular}
\caption{NMI with tuning $\lambda$ on SBM}
\label{T1}
\end{subtable}%
\begin{subtable}[t]{0.50\textwidth}
\centering
\begin{tabular}{|l|l|l|l|l|}
\hline
                               & Truth & MATR-CV & ECV& BH \\
\hline                               
Football                       & 12    & 12      & 10    & 10 \\
\hline
Polblogs                       & 2     & 6       & 1     & 8  \\ \cline{1-1}
\hline
\multicolumn{1}{|l|}{Polbooks} & 3     & 6       & 2     & 4  \\ \cline{1-1}
\hline
\end{tabular}
\caption{Model selection with SBM}
\label{T2}
\end{subtable}%
\caption{Results obtained on real networks}
\end{table}
\vspace{-1em}
\subsection{MATR on mixture model with known number of clusters}\label{sec:exp_matr_mixture}

We use MATR-CV to select the bandwidth parameter $\theta$ in spectral clustering applied to mixture data when given the correct number of clusters. In all the examples, our candidate set of $\theta$ is $\{t\alpha/20\}$ for $t=1,\cdots, 20$ and $\alpha=\max_{i,j}\|Y_i-Y_j\|_2$. We compare MATR with three other well-known heuristic methods. The first one was proposed by \citep{shi2008data} (DS), where, for each data point $Y_i$, the $5\%$ quantile of $\{\norm{Y_i-Y_j}_2, j=1,...,n\}$ is denoted $q_i$ and then $\theta$ is set to be $\frac{95\% \text{ quantile of } \{q_1,...,q_n\}}{\sqrt{95\% \text{ quantile of } \chi_d^2}}$. We also compare with two other methods in  \cite{von2007tutorial}: a method based on $k$-nearest neighbor (KNN) and a method based on minimal spanned tree (MST). For KNN, $\theta$ is chosen in the order of the mean distance of a point to its $k$-th nearest neighbor, where $k \sim \log(n)+1$. For MST, $\theta$ is set as the length of the longest edge in a minimal spanning tree of the fully connected graph on the data points. 

\noindent\textbf{Simulated data.}
We first conduct experiments on simulated data generated from a 3-component Gaussian mixture with $d=20$. 
  The means are multiplied by a separation constant which controls clustering difficulty (larger, the better).  Detailed descriptions of the parameter settings can be found in Section~\ref{sec:expdetail} of the \supp. $n=500$ datapoints are generated for each mixture model and random runs are used to calculate standard deviations for each parameter setting. In Figure~\ref{fig:tuning} (A) and (B) we plot NMI on the $Y$ axis against the separation along the $X$ axis for mixture models with equal and unequal mixing coefficients respectively. For all these settings, MATR performs as well or better than the best among DS, KNN and MST. 

\noindent\textbf{Real data.}
We also test MATR for tuning $\theta$ on a real dataset: Optical Recognition of Handwritten Digits Data Set\footnote{https://archive.ics.uci.edu/ml/datasets/Optical+Recognition+of+Handwritten+Digits}. We use a copy of the test set provided by scikit-learn \citep{scikit-learn}, which consists of 1797 instances of 10 classes. We standardize the dataset before clustering. With $10$ clusters, MATR, DS, KNN and MST yield cluster results with NMI values $0.64$, $0.45$, $0.64$ and $0.62$ respectively. In other words, MATR performs similarly to KNN but outperforms DS and MST. We also visualize and compare the clustering results by different methods in 2-D using tSNE \citep{maaten2008visualizing}, which can be found in Section~\ref{sec:expdetail} of the \supp.

\subsection{Model selection with MATR-CV on SBM}\label{exp3}
We make comparisons among MATR-CV, Bethe-Hessian estimator (BH) \citep{le2015estimating} and ECV \citep{li2016network}. For ECV and MATR-CV, we consider $r\in\{1,\cdots \sqrt{n}\}$, where $n$ is the number of nodes.

\noindent\textbf{Simulated data.}
We simulate networks from a $4$-cluster strongly assortative SBM with equal and unequal sized blocks (detailed in Section~\ref{sec:expdetail}  of the \supp). 
In Figure \ref{fig:matr-cv}, we show NMI on $Y$ axis vs. average degree on $Y$ axis.   In Figure \ref{fig:matr-cv}(a) and (b) we respectively consider equal sized ($4$ clusters of size $100$) and unequal sized networks (two with $120$ nodes and two with $80$ nodes). In all cases, MATR-CV has the highest NMI. 
A table with median number of clusters selected by each method can be found in Section~\ref{sec:expdetail} of the \supp.

\noindent\textbf{Real data.}
The same set of methods are also compared on three real datasets: the football dataset, the political blogs dataset and the political books dataset. The results are shown in Table \ref{T2}, where MATR-CV finds the ground truth for the football dataset.
\begin{figure}[htbp!]
\begin{subfigure}[t]{.5\textwidth}
  \centering
  \includegraphics[width=.6\linewidth]{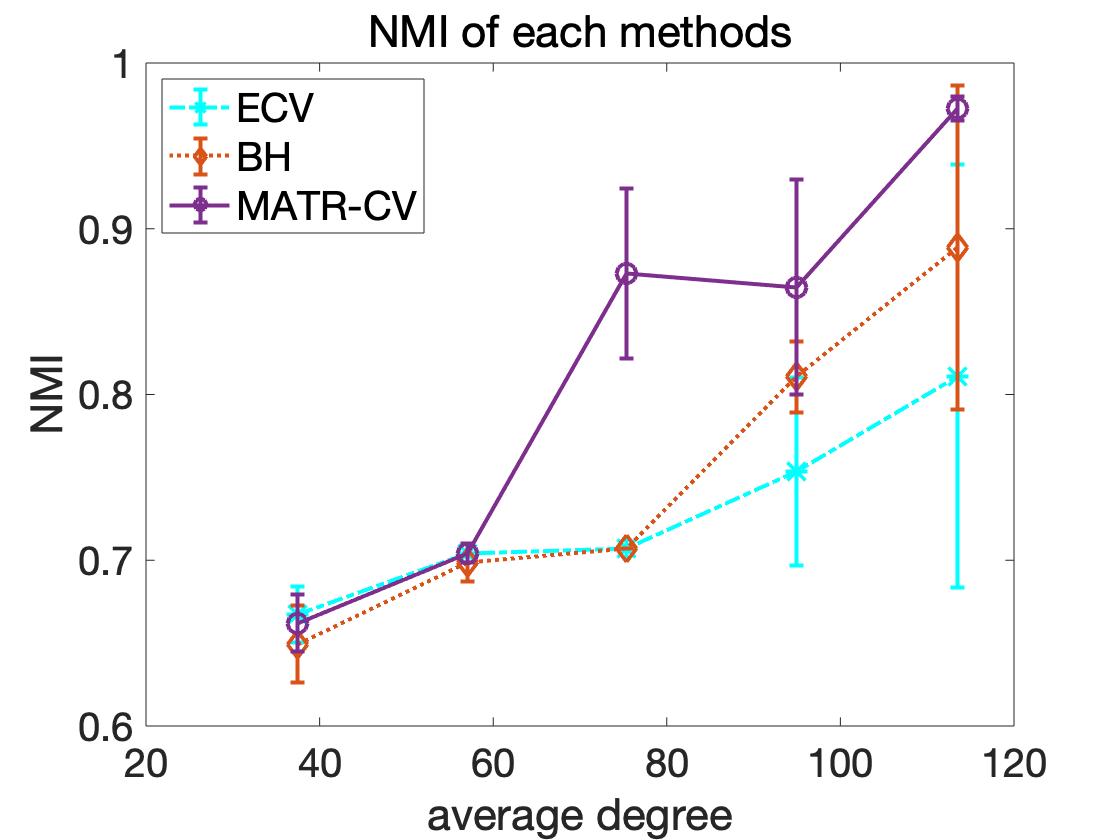} 
  \caption{NMI for equal sized case}
\end{subfigure}
\begin{subfigure}[t]{.5\textwidth}
  \centering
  \includegraphics[width=.6\linewidth]{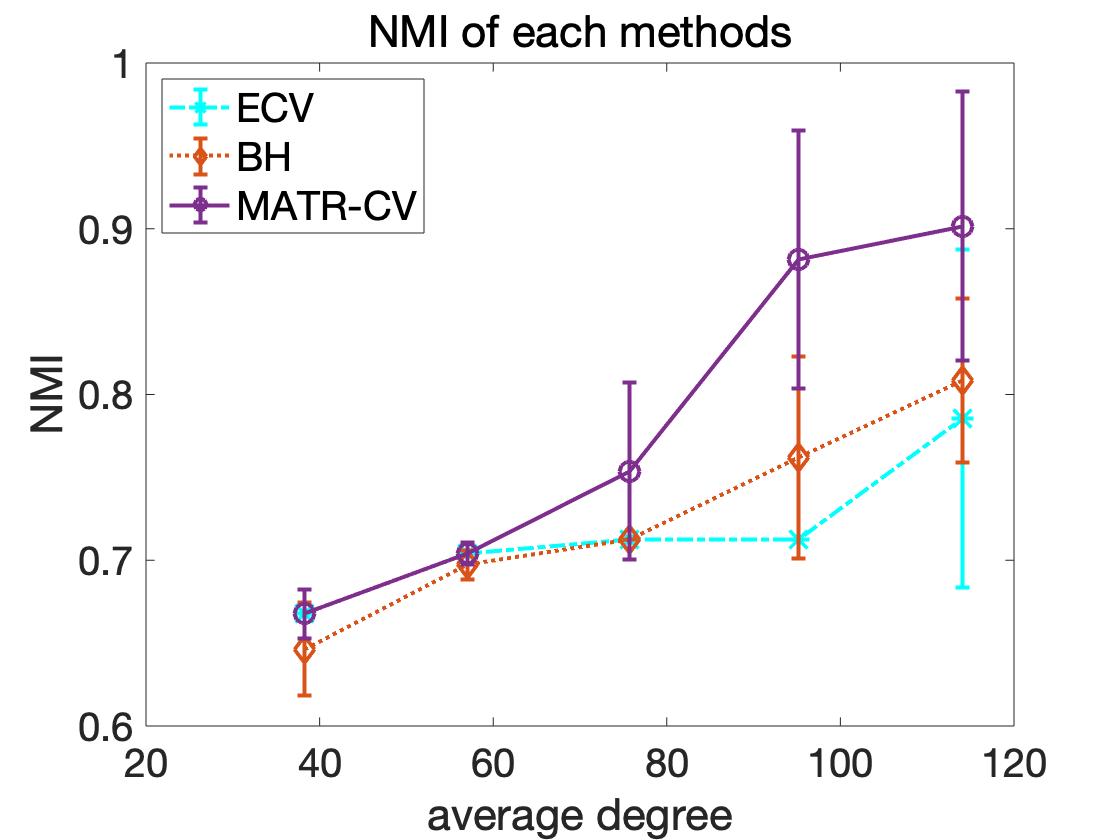} 
  \caption{NMI for unequal sized case}
\end{subfigure}\\
\caption{Comparison of NMI with model selection for equal and unequal sized cases.}
\label{fig:matr-cv}
\end{figure}
\subsection{Model selection with MATR-CV on MMSB}\label{sec:exp_matrcv_mmsb}
We compare MATR-CV with Universal Singular Value Thresholding (USVT) \citep{chatterjee2015matrix}, ECV \citep{li2016network}  and SIMPLE \citep{fan2019simple} in terms of doing model selection with MMSB. For ECV and MATR-CV, we consider the candidate set $r\in\{1,2,\cdots, \lfloor \hat{\rho} n \rfloor \}$, where $\hat{\rho}=\sum_{i<j} A_{ij}/{n\choose 2}$.

\noindent\textbf{Simulated data.}
We first apply all the methods to simulated data. We consider $B=\rho\times\{(p-q) I_{r}+ q E_{r}\}$. Following \citep{mao2018overlapping}, we sample $\Theta_i \sim \text{Dirichlet}(\alphav)$ and $\alphav = \mathbf{1}_{r} / r$. We generate networks with $n=2000$ nodes with $r=4$ and $r=8$ respectively. We set $p=1, q=0.1$ when $r=4$; $p=1, q=0.01$ when $r=8$ for a range of $\rho$.
In Table~\ref{tab:MATR_CV_cluster_4} and~\ref{tab:MATR_CV_cluster_8}, we report the fractions of exactly recovering the true cluster number $r$ over 40 runs for each method across different average degrees. We observe that in both $r=4$ and $r=8$ cases, MATR-CV outperforms the other three methods with a large margin on sparse graphs. The method SIMPLE consistently underfits in our sparsity regime, which is  understandable, since their theoretical guarantees hold for a dense degree regime.

 \begin{table}[htb]
\begin{subtable}[t]{0.50\textwidth}
\centering
\begin{adjustbox}{width=\columnwidth,center}
\begin{tabular}[t]{|l|l|l|l|l|l|l|}
\hline
$\rho$   & 0.01 & 0.03  & 0.06 & 0.08 & 0.11 & 0.13 \\
\hline       
MATR-CV & 0.35    &  0.83 &  0.93 & 1 & 1 & 1 \\
\hline
USVT & 0  & 0&  1     & 1  & 1 & 1 \\
\hline
ECV  & 0    & 0 & 0   & 0.95 & 1 & 1\\
\hline
\end{tabular}
\end{adjustbox}
\caption{Exact recovery fractions for $4$ clusters}
\label{tab:MATR_CV_cluster_4}
\end{subtable}%
\begin{subtable}[t]{0.50\textwidth}
\centering
\begin{adjustbox}{width=\columnwidth,center}
\begin{tabular}[t]{|l|l|l|l|l|l|l|}
\hline
$\rho$   & 0.02 & 0.05   & 0.09 & 0.12 & 0.16 & 0.21\\
\hline       
MATR-CV & 0.10   & 0.43 & 0.95     & 0.93 & 0.95 & 1   \\
\hline
USVT & 0  & 0& 0.58     & 1  & 1 & 1\\
\hline
ECV  & 0   & 0 & 0    & 0.93 & 1 & 1\\
\hline
\end{tabular}
\end{adjustbox}
\caption{Exact recovery fractions for $8$ clusters}
\label{tab:MATR_CV_cluster_8}
\end{subtable}%
\caption{Results of MMSB on synthetic data}
\end{table}




\noindent\textbf{Real data.}
We also test MATR-CV with MMSB on a real network, the political books network, which contains 3 clusters. Here fitting a MMSB model is reasonable since each book can have mixed political inclinations, e.g. a ``conserved'' book may be in fact mixed between ``neutral'' and ``conservative''. With MATR-CV, we found $3$ clusters. With USVT,  ECV and SIMPLE we found fewer than $3$ clusters. 

\section{Discussion}
\label{sec:discussion}

Clustering data, both in i.i.d and network structured settings have received a lot of attention both from applied and theoretical communities. However, methods for tuning hyperparameters involved in clustering problems are mostly heuristic. In this paper, we present MATR, a provable MAx-TRace based hyperparameter tuning framework for general clustering problems. We prove the effectiveness of this framework for tuning SDP relaxations for community detection under the block model and for learning the bandwidth parameter of the gaussian kernel in spectral clustering over a mixture of subgaussians. Our framework can also be used to do model selection using a cross validation based extension (MATR-CV) which can be used to consistently estimate the number of clusters in blockmodels and mixed membership blockmodels. Using a variety of simulation and real experiments we show the advantage of our method over other existing heuristics. 

The framework presented in this paper is general and can be applied to doing model selection or tuning for broader model classes like degree corrected blockmodels ~\citep{karrer2011dcbm}, since there are many exact recovery based algorithms for estimation in these settings~\citep{chen2018}.  We believe that our framework can be extended to the broader class of degree corrected mixed membership blockmodels~\citep{jin2017estimating} which includes the topic model~\citep{mao2018overlapping}. However, the derivation of the estimation error $\epsilon_{\text{est}}$ involves tedious derivations of parameter estimation error, which has not been done by existing works. 
Furthermore, even though our work uses node sampling, we believe we can extend the MATR-CV framework to get consistent model selection  for other sampling procedures like edge sampling~\citep{li2016network}. 


\appendix
\section*{}
\begin{center}
{\large\bf Appendix}
\end{center}
This appendix contains detailed proofs of theoretical results in the main paper  ``A Unified Framework for Tuning Hyperparameters in Clustering Problems'', additional theoretical results, and detailed description of the experimental parameter settings. 
We present proofs for MATR and MATR-CV in Sections~\ref{sec:proofformatr} and Sections~\ref{sec:proofformatrcv} respectively. Sections~\ref{sec:merge1} 
also contains additional theoretical results on the role of the hyperparameter in merging clusters in~\ref{SDP:YD} and~\ref{SDP:BW} respectively.
Finally, Section~\ref{sec:expdetail} contains detailed parameter settings for the experimental results in the main paper.

\section{Additional theoretical results and proofs of results in Section~\ref{sec:knownk}}\label{sec:proofformatr}
\subsection{Proof of Theorem \ref{thm:general_setup}}
\label{sec:thm1}

\begin{proof}
If for tuning parameter $\lambda$, we have $\langle \hat{S}, \hat{X}_{\lambda} \rangle \geq \langle S, X_0 \rangle -\epsilon$, then 
	\begin{align}\label{eq:assump}
	    \langle S, \hat{X}_{\lambda} \rangle \geq \langle S, X_0 \rangle - |\langle \hat{S} - S, \hat{X}_{\lambda} \rangle|- \epsilon.
	\end{align} 	
	First we will prove that this immediately gives an upper bound on $\|\hat{X}_\lambda-X_0\|_F$.
	We will remove the subscript $\lambda$ for ease of exposition.
	Denote $\omega_k=\langle X_0, \hat{X}_{C_k, C_k}\rangle$, $\alpha_{ij} = \frac{\langle E_{i, j}, \hat{X}\rangle}{m_k(1-\omega_k)},$ when $\omega_k<1$ and $0$ otherwise, and off-diagonal set for $k$th cluster $C_{k}^c$ as $\{(i,j)|i\in C_k, j\notin C_k\}$. Then we have 
	\begin{equation}
	\begin{split}
	\langle S, \hat{X} \rangle &=\sum_{k=1}^{r_0} a_{kk} \langle E_{C_k, C_k}, \hat{X} \rangle + \sum_{k=1}^{r_0}\sum_{(i,j)\in C_{k}^c} a_{ij} \langle E_{i,j}, 
	\hat{X} \rangle\\
	&=\sum_{k=1}^{r_0} a_{kk} m_k \omega_k + \sum_{k=1}^{r_0} m_k (1-\omega_k) \sum_{(i,j)\in C_{k}^c} a_{ij} \alpha_{ij}\\
	&=\sum_{k=1}^{r_0} m_k \omega_k(a_{kk}-\sum_{(i,j)\in C_{k}^c} a_{ij}\alpha_{ij}) + \sum_{k=1}^{r_0} m_k \sum_{(i,j)\in C_{k}^c} a_{ij} \alpha_{ij}
	\end{split}
	\end{equation}
	
	Since $\langle S, X_0\rangle=\sum_k m_k a_{kk}$, by~\eqref{eq:assump}, $\langle S, \hat{X} \rangle \geq \sum_k m_k a_{kk} -|\langle R, \hat{X} \rangle|- \epsilon$, we have
	$$\sum_k m_k \omega_k(a_{kk}-\sum_{(i,j)\in C_{k}^c} a_{ij}\alpha_{ij}) + \sum_k m_k \sum_{(i,j)\in C_{k}^c} a_{ij} \alpha_{ij} \geq \sum_k m_k a_{kk} -|\langle R, \hat{X} \rangle|- \epsilon.$$
	
	Note that, since $S$ is weakly assortative, $a_{kk}-\sum_{(i,j)\in C_{k}^c} a_{ij}\alpha_{ij}$ is always positive because $\sum_{(i,j)\in C_{k}^c} \alpha_{ij} \leq 1$.
	
	Denote $\epsilon' =|\langle R, \hat{X} \rangle|+\epsilon $, $\beta_k = \frac{m_k (a_{kk}- \sum_{C_{k}^c} \alpha_{ij}a_{ij})}{\sum_k m_k (a_{kk} - \sum_{C_{k}^c} \alpha_{ij}a_{ij})}$,
	\begin{align*}
	   \sum_k m_k \omega_k(a_{kk}-\sum_{(i,j)\in C_{k}^c} a_{ij}\alpha_{ij})&\geq \sum_k m_k (a_{kk} - \sum_{(i,j)\in C_{k}^c} \alpha_{ij}a_{ij}) -\epsilon'\\
	   \sum_k \beta_k \omega_k&\geq 1 -\frac{\epsilon'}{ \sum_k m_k (a_{kk} - \sum_{C_{k}^c} \alpha_{ij}a_{ij}) }\\
	   \sum_k \beta_k (1-\omega_k) &\leq \frac{\epsilon'}{ \sum_k m_k (a_{kk} - \sum_{C_{k}^c} \alpha_{ij}a_{ij}) }.\\
	   \sum_k (1-\omega_k) &\leq \sum_k \frac{\beta_k}{\beta_{\min}} (1-\omega_k)\leq \frac{\epsilon'}{ \beta_{\min}\sum_k m_k (a_{kk} - \sum_{C_{k}^c} \alpha_{ij}a_{ij}) },
	\end{align*}
	where $\beta_{\min}=\min_k\beta_k$.
	Since $\text{trace}(\hat{X}) = \text{trace}(X_0)$,
	\begin{equation}
	\begin{aligned}
	\norm{\hat{X} - X_0}_F^2 &= \text{trace}((\hat{X}-X_0)^T(\hat{X}-X_0))\nonumber\\
	&= \text{trace}(\hat{X}+X_0 - 2\hat{X}X_0)\nonumber\\
	&= 2\text{trace}(X_0) - 2\sum_{k}\langle X_0,\hat{X}_{C_k,C_k}\rangle\nonumber\\
	&= 2 \sum_k(1-\omega_k)\leq \frac{2\epsilon'}{\min_k {m_k (a_{kk} - \sum_{C_{k}^c}   \alpha_{ij}a_{ij})}}\nonumber  \\
	&\leq  \frac{2\epsilon'}{n\pi_{\min}\min_k (a_{kk} - \max_{C_{k}^c}   a_{ij})} = \frac{2\epsilon'}{\tau}. \label{eq:normub}
	\end{aligned}
	\end{equation}
	Now consider the $\lambda_*$ returned by MATR, 
	
	$$	\langle \hat{S}, \hat{X}_{\lambda_*} \rangle \geq \langle \hat{S}, \hat{X}_{\lambda} \rangle \geq \langle S, X_0 \rangle - \epsilon.$$
	Then, following the above argument and from the condition from the theorem, 
	$$\|X_{\lambda_*}-X_0\|_F^2\leq \frac{2\epsilon'}{n\pi_{\min}\min_k (a_{kk} - \max_{C_{k}^c}   a_{ij})} \leq \frac{2}{\tau}(\epsilon+\sup_{X\in \mathcal{X}_{r_0}} |\langle X, R\rangle |).$$
	
\end{proof}

\subsection{Range of $\lambda$ for merging clusters in \ref{SDP:YD}}
\label{sec:merge1}
\begin{Proposition}\label{prop:sdp1merge}
	Let $\tilde{X}$ be the optimal solution of \ref{SDP:YD} for $A\sim SBM(B,Z_0)$ with $\lambda$ satisfying
	\begin{equation}
	\begin{split}
	\max_{k\neq \ell}B^*_{k,\ell} + \Omega(\sqrt{\frac{\rho \log n}{n\pi_\text{min}}})\leq\lambda  \leq \min_k B_{kk}^* -\max_{k, \ell=r-1,r} \frac{m_{\ell}}{n_k}(B_{\ell,\ell}-B_{r,r-1})+ O(\sqrt{\frac{\rho \log n}{n\pi_{\max}^2}}),
	\end{split}\nonumber
	\end{equation}
	then $\tilde{X}=X^*$ with probability at least $1-\frac{1}{n}$, where $X^*$ is the unnormalized clustering matrix which merges the last two clusters, $B^*$ is the  corresponding $(r-1)\times(r-1)$ block probability matrix. 
\end{Proposition}
{\bf Remark: }The proposition implies if the first $r-2$ clusters are more connected within each cluster than the last two clusters and the connection between first $r-2$ clusters and last two clusters are weak, we can find a range for $\lambda$ that leads to merging the last two clusters with high probability. The results can be generalized to merging several clusters at one time. The result above highlights the importance of selecting $\lambda$ as it affects the performance of \ref{SDP:YD} significantly. 

\begin{proof}
	We develop sufficient conditions with a contruction of the dual certificate which guarantees $X^*$ to be the optimal solution. 
	The KKT conditions can be written as below:
	
	First order stationary:
	$$-A -\Lambda +\lambda E_n-\text{diag}(\beta) - \Gamma=0$$
	Primal feasibility:
	$$X\succeq 0,  X\geq 0, X_{ii} = { 1}\quad \forall i=1\cdots,n$$
	Dual feasibility:
	$$\Gamma\geq 0, \Lambda \succeq 0$$
	Complementary slackness
	$$\langle \Lambda, X \rangle =0, \Gamma \circ X=0.$$

	Consider the following construction:
	denote $T_k=C_k, n_k=m_k,$ for $k< r-1$, $T_{r-1}=C_{r-1}\bigcup C_{r}, n_{r-1}=m_{r-1}+m_{r}$.
	$$X_{T_k}=E_{n_k}$$
	$$X_{T_k T_l}=0,\text{for } k\neq l\leq r-1$$

	$$\Lambda _{T_k}=-A_{T_k}+\lambda E_{n_k} - \lambda n_k I_{n_k} + \text{diag}(A_{T_k}{\bf 1}_{n_k})$$
	
	$$\Lambda_{T_k T_l}=-A_{T_k, T_l} + \frac{1}{n_l}A_{T_k, T_l}E_{n_l} + \frac{1}{n_k} E_{n_k}A_{T_k T_l} - \frac{1}{n_ln_k}E_{n_k}A_{T_k, T_l}E_{n_l}$$
	$$\Gamma_{T_k}=0$$
	$$\Gamma_{T_k, T_l}=\lambda E_{n_k, n_l}- \frac{1}{n_l}A_{T_k, T_l}E_{n_l} - \frac{1}{n_k}E_{n_k}A_{T_k T_l}+\frac{1}{n_ln_k}E_{n_k}A_{T_k, T_l}E_{n_l}$$
	$$\beta = \text{diag}(-A -\Lambda +\lambda E_n - \Gamma)$$
	
	All the KKT conditions are satisfied by construction except for positive semidefiniteness of $\Lambda$ and positiveness of $\Gamma$. Now, we show it one by one.
	
	{\bf Positive Semidefiniteness of $\Lambda$}  \space \space Since $\text{span}(1_{T_k})\subset \text{ker}(\Lambda)$, it suffices to show that for any $u\in span(1_{T_k})^\perp, u^T\Lambda u\geq 0$. Consider $u=\sum_k u_{T_k},$ where $u_{T_k}:=u\circ 1_{T_k}$, then $u_{T_k}\perp 1_{n_k}$.
	\begin{equation}
	\begin{split}
	u^T\Lambda u=&-\sum_k u_{T_k}^TA_{T_k}u_{T_k} - \lambda \sum_k n_k u_{T_k}^T u_{T_k}  + \sum_k u_{T_k}^T \text{diag}(A_{T_k}{\bf 1}_{n_k})u_{T_k}  - \sum_{k\neq l}u_{T_k}^TA_{T_k T_l}u_{T_l}\\
	& =-u^T (A-P)u^T - u^T P u - \lambda \sum_k n_k u_{T_k}^T u_{T_k}  + \sum_k u_{T_k}^T \text{diag}(A_{T_k}{\bf 1}_{n_k})u_{T_k}  \\
	&=-u^T(A-P)u -u_{T_{k-1}}^TP_{T_{k-1}T_{k-1}}u_{T_{k-1}}  - \lambda \sum_k n_k u_{T_k}^T u_{T_k}  + \sum_k u_{T_k}^T \text{diag}(A_{T_k}{\bf 1}_{n_k})u_{T_k} \\
	\end{split}\label{eq:pslambda1}
	\end{equation}
	For the first term, we know
	$$u^T(A-P)u \leq \norm{A-P}_2 \norm{u}_2^2\leq O(\sqrt{n\rho})\norm{u}_2^2 $$ with high probability.
	
	For the second term, and note that $T_{r-1}=C_{r-1}\bigcup C_{r}$, and $$P_{T_{r-1}T_{r-1}} = \begin{bmatrix}
	B_{r-1,r-1}E_{m_{r-1}m_{r-1}}, B_{r-1,r}E_{m_{r-1}m_{r}}\\
	B_{r,r-1}E_{m_rm_{r-1}}, B_{r,r}E_{m_{r}m_{r}}
	\end{bmatrix}$$
	Since $u_{T_{r-1}}\perp 1_{n_{r-1}}$, $$u_{T_{r-1}}^T \begin{bmatrix}
	B_{r-1,r}E_{m_{r-1}m_{r-1}}, B_{r-1,r}E_{m_{r-1}m_{r}}\\
	B_{r,r-1}E_{m_rm_{r-1}}, B_{r,r-1}E_{m_{r}m_{r}}
	\end{bmatrix} u_{T_{r-1}}=0,$$
	therefore 
	\begin{equation}
	\begin{split}
	u_{T_{r-1}}^TP_{T_{r-1}T_{r-1}}u_{T_{r-1}}&=u_{T_{r-1}}^T \begin{bmatrix}
	(B_{r-1,r}-B_{r-1,r-1})E_{m_{r-1}m_{r-1}},0\\
	0, (B_{r-1,r}-B_{r,r})E_{m_rm_{r}}
	\end{bmatrix}u_{T_{r-1}}\\
	&\leq \max \{m_{r-1}(B_{r-1,r-1}-B_{r-1,r}), m_{r}(B_{r,r}-B_{r,r-1})\}\norm{u}_2^2
	\end{split}
	\end{equation}
	
	Consider the last term $ \sum_k u_{T_k}^T \text{diag}(A_{T_k}{\bf 1}_{n_k})u_{T_k} $. Using Chernoff, we know
	$$||\text{diag}(A_{T_k} {\bf 1}_{n_k})||_2 \geq B_{k,k}^* n_k - \sqrt{6\rho n_k\log n_k}$$ with high probability,
	where for $k, l<r-1$, $$B_{kl}^*=B_{kl},$$ 
	$$B_{k,r-1}^*=\frac{m_{r-1}B_{k,r-1}+m_r B_{k,r}}{m_{r-1}+m_r},$$
	$$B_{r-1,r-1}^*=\frac{(m_{r-1}^2B_{r-1,r-1}+2*m_r m_{r-1}B_{r-1,r}+(m_{r}^2B_{r,r})}{(m_{r-1}+m_r)^2}.$$
	
	Therefore, :
	$$- \lambda \sum_k n_k u_{T_k}^T u_{T_k}  + \sum_k u_{T_k}^T \text{diag}(A_{T_k}{\bf 1}_{n_k})u_{T_k}\geq \min_k({B_{k,k}^*n_k-\Omega(\sqrt{\rho n_k\log n}) - \lambda n_k}) \norm{u}_2^2.$$
	
	So with equation~\ref{eq:pslambda1}, a sufficient condition for positive semidefiniteness of $\Lambda$ is $$\min_k({B_{k,k}^*n_k-\Omega(\sqrt{\rho n_k\log n}) - \lambda n_k}) \geq O(\sqrt{n\rho}) +\max  \{m_{r-1}(B_{r-1,r-1}-B_{r-1,r}), m_{r}(B_{r,r}-B_{r,r-1})\}$$
	
	which implies,  $$\lambda \leq \min_k B_{kk}^* -\max_k \max \{\frac{m_{r-1}}{n_k}(B_{r-1,r-1}-B_{r-1,r}), \frac{m_{r}}{n_k}(B_{r,r}-B_{r,r-1})\}+ O( \sqrt{\rho \log n/n\pi_{\max}^2})$$

	{\bf Positiveness of $\Gamma$}  For $i\in T_k, j\in T_l$, we have
	$$
	\Gamma_{i,j}=\lambda -\frac{\sum_{m\in T_l}A_{i,m}}{n_l}-\frac{\sum_{m\in T_k}A_{m,j}}{n_k}+\frac{1}{n_kn_l}\sum_{m\in T_k,o\in T_l} A_{mo}. 
	$$
	Therefore, block-wise mean of $\Gamma$ will be
	$$\mathbb{E}[\Gamma_{T_k, T_l}]=(\lambda - B^*_{k,l})E_{n_k, n_l},$$
	and the variance for each entry belonging to cluster $k$ and $l$ will be in order of $O(\rho/(n_kn_l))$.
	
	Using Chernoff bound, we have
	$$
	p(|\Gamma_{i,j}-(\lambda-B^*_{k,l})|>\lambda-B^*_{k,l})\leq 2\exp\bigg[{-\frac{n_kn_l}{2\rho}(\lambda-B^*_{k,l})^2}\bigg].
	$$
	
	Therefore, as long as $\lambda \geq \max_{k\neq l}B^*_{k,l} + \Omega(\sqrt{\rho \log n /n\pi_\text{min}})$, we have
	
	$$
	p(\Gamma_{i,j}<0)\leq2\exp\bigg[-\frac{n\pi_{min}\log n}{2}\bigg]
	$$
	We then applying the union bound and conclude that $\Gamma_{T_kT_l}>0$ with a high probability when $\lambda \geq \max_{k\neq l}B^*_{k,l} + \Omega(\sqrt{\rho \log n /n\pi_\text{min}})$.

\end{proof}

\begin{proposition}\label{prop:sdp1recover}
	As long as $ \max_{k\neq l}B_{k,l} + \Omega(\sqrt{\rho \log n /n\pi_\text{min}}) \leq \lambda \leq \min_k B_{kk}+ O( \sqrt{\rho \log n/n\pi_{\max}^2})$, ~\ref{SDP:YD} exactly recovers $X_0$ with high probability.
\end{proposition} 

\begin{proof}
We follow the same primal-dual construction as Proposition \ref{prop:sdp1merge} without merging the last two clusters. Consider the following construction: denote $T_k=C_k, n_k=m_k,$ for $k=1,...,r$. We show the positive semidefiniteness and Positiveness of $\Lambda$ and $\Gamma$ respectively. 

{\bf Positve Semidefiniteness of $\Lambda$}   \space \space
    Since $\text{span}(1_{T_k})\subset \text{ker}(\Lambda)$, it suffices to show that for any $u\in \text{span}(1_{T_k})^\perp, u^T\Lambda u\geq 0$. Consider $u=\sum_k u_{T_k},$ where $u_{T_k}:=u\circ 1_{T_k}$, and $u_{T_k}\perp 1_{n_k}$, we have
	\begin{equation}
	\begin{split}
	u^T\Lambda u=&-\sum_k u_{T_k}^TA_{T_k}u_{T_k} - \lambda \sum_k n_k u_{T_k}^T u_{T_k}  + \sum_k u_{T_k}^T \text{diag}(A_{T_k}{\bf 1}_{n_k})u_{T_k}  - \sum_{k\neq l}u_{T_k}^TA_{T_k T_l}u_{T_l}\\
	& =-u^T (A-P)u^T - u^T P u - \lambda \sum_k n_k u_{T_k}^T u_{T_k}  + \sum_k u_{T_k}^T \text{diag}(A_{T_k}{\bf 1}_{n_k})u_{T_k}  \\
	&=-u^T(A-P)u - \lambda \sum_k n_k u_{T_k}^T u_{T_k}  + \sum_k u_{T_k}^T \text{diag}(A_{T_k}{\bf 1}_{n_k})u_{T_k} \\
	\end{split}
	\end{equation}    
	For the first term, we know
	$$u^T(A-P)u \leq \norm{A-P}_2 \norm{u}_2^2\leq O(\sqrt{n\rho})\norm{u}_2^2 $$ with high probability, and using Chernoff, we have
	$$||\text{diag}(A_{T_k} {\bf 1}_{n_k})||_2 \geq B_{k,k} n_k - \sqrt{6\rho n_k\log n_k}$$ with high probability. 
	Therefore, 
	$$- \lambda \sum_k n_k u_{T_k}^T u_{T_k}  + \sum_k u_{T_k}^T \text{diag}(A_{T_k}{\bf 1}_{n_k})u_{T_k}\geq \min_k({B_{k,k}n_k-\Omega(\sqrt{\rho n_k\log n}) - \lambda n_k}) \norm{u}_2^2,$$	
	
	which implies a sufficient condition for positive semidefiniteness of $\Lambda$ is 
    $$\lambda \leq \min_k B_{kk}+ O( \sqrt{\rho \log n/n\pi_{\max}^2}),$$
    and the lower bound can be obtained exactly the same way as Proposition \ref{prop:sdp1merge}.
	Using Chernoff bound,  $\Gamma_{T_k T_l} > 0$ with high probability as long as $\lambda \geq \max_{k\neq l}B_{k,l} + \Omega(\sqrt{\rho \log n /n\pi_\text{min}})$.   
    
\end{proof}

\subsection{Proof of Corollary~\ref{thm:fixk}}
\begin{proof}
	This result comes directly from Theorem \ref{thm:general_setup}. 
	We have $S=\tilde{P}$, $R = (A-P)+(P-\tilde{P})$. For $\lambda_0$,
	$$
	\langle \hat{X}_{\lambda_0}, A\rangle \geq \langle X_{0}, \tilde{P}\rangle -O(r\rho) -\epsilon,
	$$
	where $r\rho=o(\tau)$ since $r\sqrt{n\rho}=o(\tau)$, and for any $\hat{X}\in\mathcal{X}_r$,
	$$	|\langle A-\tilde{P},\hat{X}\rangle| \leq ||A-P||_{op} \tr{\hat{X}} + O(r\rho)= O_P(r\sqrt{n\rho}).
	$$
	The last inequality follows by~\cite{lei2015} and $n\rho \geq c\log n$.
\end{proof}

\subsection{Proof of Corollary~\ref{cor:mixture}}
\begin{proof}
	
	First note that $p_{gap}$ as defined in Eq~\ref{eq:pgap} is $\delta_{\text{sep}}^2/2$, where $\delta_{\text{sep}}$ is the minimum Euclidean distance between two cluster centers. Using the argument in~\cite{mixon2017sdp} and Theorem~\ref{thm:general_setup} we obtain:
	\begin{align*}
	\|\hat{X}_{\theta^*}-X_0\|_F^2&\leq 4 \frac{n\pi_{\text{min}}\epsilon+\sup_{X\in \mathcal{X}_r} |\langle X, \hat{S}-S\rangle |}{n\pi_{\min}\delta_{\text{sep}}^2}\\
	&\stackrel{(i)}{\leq} C\frac{\epsilon+r\alpha\sigma_{\max}^2(\alpha+\min\{r,d\})}{\delta_{\text{sep}}^2}
	\end{align*}
	Step $(i)$ is true with probability at least $1-2\eta$, as long as $n\geq \max\{c_1 d, c_2\log (2/\eta),\log (c_4/\eta)\}$, using the argument from Theorem 2 in~\cite{mixon2017sdp}.
\end{proof}

\section{Additional Theoretical Results and Proofs of Results in Section~\ref{sec:uknownk}}
\label{sec:proofformatrcv}

\subsection{Proof of Theorem~\ref{thm:matrcv_general}}

\begin{proof}
With probability greater than $1-\delta_\text{est}-\delta_\text{over}-\delta_\text{under}$, the following three inequalities hold.

For $r_T\geq r_t > r:$
\begin{align*}
    \langle \hat{S}^{22}, \hat{X}^{22}_{r} \rangle& \geq  \langle \hat{S}^{22}, X_0^{22} \rangle -\epsilon_\text{est} \geq \max_{r_T\geq r_t >r} \langle \hat{S}^{22}, \hat{X}^{22}_{r_t}\rangle  -\epsilon_\text{est} - \epsilon_\text{over}\\
    & \geq \max_{r_T\geq r_t >r}\langle \hat{S}^{22}, \hat{X}^{22}_{r_t}\rangle - \Delta  \\
\end{align*}

For $r_t < r:$
\begin{align*}
    \langle \hat{S}^{22}, \hat{X}^{22}_{r} \rangle& \geq \langle \hat{S}^{22}, X_0^{22} \rangle -\epsilon_\text{est} \geq \max_{ r_t <r}\langle \hat{S}^{22}, \hat{X}^{22}_{r_t}\rangle  -\epsilon_\text{est} + \epsilon_\text{under}\\
    & > \max_{ r_t <r}\langle \hat{S}^{22}, \hat{X}^{22}_{r_t}\rangle + \Delta  \\
\end{align*}

Therefore, with probability at least $1-\delta_\text{est}-\delta_\text{over}-\delta_\text{under}$, $\langle \hat{S}^{22}, \hat{X}^{22}_{r} \rangle \geq \max_{t} \langle \hat{S}^{22}, \hat{X}^{22}_{r_t} \rangle -\Delta$. Let  $R= \{r_t: \langle \hat{S}^{22}, \hat{X}^{22}_{r_t} \rangle \geq \max_t \langle \hat{S}^{22}, \hat{X}^{22}_{r_t} \rangle - \Delta\}$. It follows then $r\in R$.

Furthermore, with probabiltiy at least $1-\delta_\text{est}-\delta_\text{under}$, for $r_t < r:$
\begin{align*}
    \langle \hat{S}^{22}, \hat{X}^{22}_{r_t} \rangle& \leq  \langle \hat{S}^{22}, X_0^{22}\rangle - \epsilon_\text{under} \leq \langle \hat{S}^{22}, \hat{X}^{22}_{r}\rangle  +\epsilon_\text{est} - \epsilon_\text{under}\\
    & < \max_t \langle \hat{S}^{22}, \hat{X}^{22}_{r_t} \rangle - \Delta.
\end{align*}
Therefore, for any $r_t< r$, $r_t\notin R$, and $\min\{r_t: r_t\in R\}=r$.
\end{proof}

\subsection{Proof of Theorem~\ref{cor:nok}}
\label{sec:cor:nok}
We first prove a concentration lemma that holds for any normalized clustering matrix $X$ independent of $A$.

\begin{lemma}\label{lem:generalsamplebound}
Consider a an adjacency matrix $A$ and its population version $P$. Let $X$ be a normalized clustering matrix independent of $A$. Then with  probability at least $1-O(n^{-1})$, $$| \langle A-P, X \rangle | \leq (1+B_{\max})\sqrt{\tr{X} \log n}$$
with $B_{\max}=\max_{i,j}B_{ij}$.
\end{lemma}
\begin{proof}
The result follows from Hoeffding's inequality and the fact that $X$ is a projection matrix.

By independence between $A$ and $X$, 
\begin{align*}
    P\left(\sum_{i<j} (A_{ij}-P_{ij})X_{ij} > t\right) & \leq \exp (-\frac{2t^2}{(1+B_{\max})^2 \sum_{i< j} X_{ij}^2})\\
    &\leq \exp (-\frac{4t^2}{ (1+B_{\max})^2\norm{X}_F^2})\\
    &=\exp (-\frac{4t^2}{ (1+B_{\max})^2\tr{X}})
\end{align*}
Let $t = \frac{1}{2}(1+B_{\max})\sqrt{\text{trace}(X) \log{n}}$, then by symmetry in $A$ and $X$, $P(\langle A-P, X \rangle > (1+B_{\max})\sqrt{\text{trace}(X) \log{n}} ) = O(1/n).$ The other direction is the same.


\end{proof}



In order to prove Theorem~\ref{cor:nok}, we need to derive the three error bounds in Theorem~\ref{thm:matrcv_general} in this setting. For notational convenience, we first derive the bounds for $A$ and a general normalized clustering matrix $\hat{X}$, with the understanding that the same asymptotic bounds apply to estimates obtained from the training graph provided the split is random and the number of training nodes is $\Theta(n)$.

\begin{lemma}\label{prop:underestimate}
For a sequence of underfitting normalized clustering matrix $\{\hat{X}_{r_t}\}_{r_t<r}$, all independent of $A$,  provided $n\rho/\sqrt{log n} \to \infty$, we have
\begin{equation*}
   \max_{r_t <r} \langle A, \hat{X}_{r_t}\rangle \leq \langle A, X_0\rangle- \Omega_P(n\rho \pi_{\text{min}}^2 /r^2),
\end{equation*}
for fixed $r$ and $\pi_{\min}$.
\end{lemma}

\begin{proof}

Let $\{\hat{C}_k\}$ be the clusters associated with $\hat{X}$. Denote $\gamma_{k,i} = |\hat{C_k} \cap C_i|$, and $\hat{m}_k = |\hat{C}_k| = \sum_i \gamma_{k,i}$, $r_t = \text{rank}(\hat{X})< r.$ First note that for each $i\in [r]$, $\exists k\in [r_t]$, s.t. $\gamma_{k,i}\geq |C_i|/r_t$. Since $r>r_t$, by the Pigeonhole principle, we see that $\exists i_0,j_0,k_0$, $i_0\neq j_0$, such that,     

$$\gamma_{k_0, i_0}=|\hat{C}_{k_0} \cap C_{i_0}| \geq |C_{i_0}| / r_t \geq \pi_\text{min} n/ r_t$$
\begin{align}
    \gamma_{k_0, j_0}=|\hat{C}_{k_0} \cap C_{j_0}| \geq |C_{j_0}| / r_t \geq \pi_\text{min} n/ r_t
    \label{eq:pigeonhole}
\end{align}

For each $k\neq k_0$, 
$$ \frac{\sum_{i,j} B_{i,j} \gamma_{k,i}\gamma_{k,j}}{\hat{m}_k} \leq \frac{\sum_{i} B_{i,i} \sum_j \gamma_{k,i}\gamma_{k,j}}{\hat{m}_k} = \sum_{i} B_{i,i}\gamma_{k,i}.$$

For $k= k_0$, 

\begin{equation}
\begin{split}
        \frac{\sum_{i,j} B_{i,j} \gamma_{k,i}\gamma_{k,j}}{\hat{m}_k} & = \frac{\sum_{i,j} B_{i,i} \gamma_{k,i}\gamma_{k,j}}{\hat{m}_k}
+ \frac{\sum_{i\neq j} (B_{i,j} - B_{i,i}) \gamma_{k,i}\gamma_{k,j}}{\hat{m}_k}\\
        & = \sum_{i} B_{i,i}\gamma_{k,i} + \frac{\sum_{i\neq j} (B_{i,j} - B_{i,i})
        \gamma_{k,i}\gamma_{k,j}}{\hat{m}_k}\\
        & \leq \sum_{i} B_{i,i}\gamma_{k,i} + \frac{(2B_{i_0,j_0} - B_{i_0,i_0} - B_{j_0,j_0}) \gamma_{k,i_0}\gamma_{k,j_0}}{\hat{m}_k}\\
        & = \sum_{i} B_{i,i}\gamma_{k,i} - \frac{\left((B_{i_0,i_0}-B_{i_0,j_0})+(B_{j_0,j_0}-B_{i_0,j_0})\right) \gamma_{k,i_0}\gamma_{k,j_0}}{\hat{m}_k}\\
        & \stackrel{(a)}{\leq} \sum_{i} B_{i,i}\gamma_{k,i} - \frac{2\tau\gamma_{k,i_0}\gamma_{k,j_0}}{n\pi_{\text{min}}\hat{m}_k}\\
        & \leq \sum_{i} B_{i,i}\gamma_{k,i} - \frac{2\tau {\pi_{\text{min}}n}}{r_t^2\hat{m}_k},
    \end{split}
\end{equation}
where $\tau= n\pi_{\text{min}}p_{\text{gap}}$, $p_{\text{gap}}:=\min_i ( B_{i,i} - \max_{j\neq i} B_{i,j})$.
$(a)$ is true by definition of $\tau$ and Eq~\eqref{eq:pigeonhole}. 

Therefore, since $\hat{m}_{k_0}\leq n$,
\begin{align}
    \langle P, \hat{X}\rangle &=  \sum_{k=1}^{r_t} \frac{\sum_{i,j} B_{i,j} \gamma_{k,i}\gamma_{k,j}}{\hat{m}_k} -O(\rho r_t)\leq \sum_{k=1}^{r_t}\sum_{i=1}^r B_{i,i}\gamma_{k,i} - \Omega(\frac{\tau {\pi_{\text{min}} n}} {r_t^2\hat{m}_{k_0}})\notag\\
    &= \langle P, X_0\rangle  - \Omega\left(\frac{\tau {\pi_{\text{min}} }} {r_t^2}\right).
    \label{eq:under_pop_sbm}
\end{align}

Next by Lemma~\ref{lem:generalsamplebound}, for each $X$ with $\tr{X}\leq r$,
\begin{align*}
|\langle A-P, X \rangle| \leq (1+B_{\max})\sqrt{r\log n}    
\end{align*}
with probabiltiy at least $1-O(1/n)$. By a union bound and using the same argument, w.h.p.
\begin{align}
\max_{r_t<r}|\langle A-P, \hat{X}_{r_t} \rangle| \leq (1+B_{\max})\sqrt{r\log n}   \label{eq:union_noise}
\end{align}

Eqs~\eqref{eq:under_pop_sbm} and~\eqref{eq:union_noise} imply w.h.p.
\begin{align}
    & \langle A, X_0\rangle - \max _{r_t<r} \langle A, \hat{X}_{r_t}\rangle    \notag\\ 
    = & \Omega\left(\frac{np_{gap} {\pi_{\text{min}}^2 }} {r^2}\right) -  O(\sqrt{r\log n}) = \Omega\left(\frac{np_{gap} {\pi_{\text{min}}^2 }} {r^2}\right) \label{eq:abound-underfit}
\end{align}
using the condition in the Lemma.

\end{proof}

\begin{lemma}\label{prop:overestimate}
For a sequence of overfitting normalized clustering matrix $\{\hat{X}_{r_t}\}_{r<r_t\leq r_T}$, all independent of $A$, $r_T=\Theta(r)$, we have w.h.p. 
$$
\max_{r<r_t\leq r_T }\langle A, \hat{X}_{r_t} \rangle \leq \langle A, X_0 \rangle + (1+B_{\max})\sqrt{r_T\log n} + B_{\max}r.
$$
\end{lemma}
\begin{proof}
First note, for any $\hat{X}$, using weak assortativity on $B$,
\begin{align}
    \langle P,\hat{X}
    \rangle & \leq \sum_{i,j} \hat{X}_{i,j} B_{C(i), C(j)}  \notag\\
    & \leq \sum_{i} B_{C(i),C(i)}\sum_j \hat{X}_{i,j}   \notag\\
    & \leq \langle P, X_0 \rangle + B_{\max}r,
\end{align}
where $C(i)$ denotes the cluster node $i$ belongs to.
By the same argument as in Eq~\eqref{eq:union_noise}, w.h.p.
\begin{align}
\max_{r<r_t\leq r_T}|\langle A-P, \hat{X}_{r_t} \rangle| \leq (1+B_{\max})\sqrt{r_T\log n}
\label{eq:union_noise_over}
\end{align}

From the above
\begin{align*}
    \max_{r<r_t\leq r_T} \langle A, \hat{X}_{r_t} \rangle &\leq \langle A, X_0 \rangle + (1+B_{\max})\sqrt{r_T\log n} + B_{\max}r. 
\end{align*}

\end{proof}

\begin{lemma}\label{prop:sdp_exactrecovery}
With probabilitiy at least  $1-O(1/n)$, MATR-CV achieves exact recovery on the testing nodes given the true cluster number $r$, i.e. $\hat{X}_r^{22}=X_0^{22}$, provided $n\pi_{\min}\rho / \log n \to \infty$, $\gamma_{\text{train}}=\Theta(1)$.
\end{lemma}

\begin{proof}
Denote $m_k^{11}$, $m_k^{22}$ as the number of nodes belonging to the cluster $C_k$ in the training graph and testing graph respectively. 

First, with Theorem 2 in \cite{yan2017provable} and Lemma~\ref{lem:skala}, we know \ref{SDP:BW} can achieve exact recovery on traininng graph with high probability. Now, consider a node $s$ in testing graph, and assume it belongs to cluster $C_k$. The probability that it is assigned to cluster $k$ is: $P(\frac{\sum_{j\in C_k} A^{21}_{s, j}}{m_k^{11}} \geq \max_{l\neq k}\frac{\sum_{j\in C_l} A^{21}_{s, j}}{m_l^{11}} )$.

Using the Chernoff bound, for some constant $c$,
\begin{align}
    P(\frac{\sum_{j\in C_k} A^{21}_{s, j}}{m_k^{11}} \geq B_{k,k} - c \sqrt{B_{k,k}\log n/m_k^{11}}) \geq 1 - n^{-3};   \notag\\
 P(\frac{\sum_{j\in C_l} A^{21}_{s, j}}{m_l^{11}} \leq B_{l,k} + c \sqrt{B_{l,k}\log n/m_l^{11}}) \geq 1 - n^{-3}.
 \label{eq:chernoff_A12}
\end{align}
Since the graph split is random, for each $k$, with probability at least $1-n^{-3}$, $| m_k^{11}-\gamma_{\text{train}}m_k| \leq c_1\sqrt{m_k\log n}$ for some constant $c_1$. By a union bound, this holds for all $k$ with probability at least $1-rn^{-3}$. Then under this event,
$$
\sqrt{\frac{B_{l,k}\log n}{m_l^{11}}} \leq c_2 \sqrt{\frac{B_{l,k}\log n}{n\pi_{\min}}}
$$
for some $c_2$ since $n\pi_{\min}/\log n \to \infty$. Since $n\pi_{\min}\rho/\log n \to \infty$, by Eq~\eqref{eq:chernoff_A12}, with probabitliy at least $1-O(rn^{-3})$,
$$
B_{k,k} - c \sqrt{B_{k,k}\log n/m_k^{11}} > \max_{l\neq k} B_{l,k} + c \sqrt{B_{l,k}\log n/m_k^{11}},
$$
and node $s$ is assigned correctly to cluster $k$. Taking a union over all $s$ in the training set, with probability 
$1-O(rn^{-2})$, MATR-CV would give exact recovery for the testing graph given $r$.

\end{proof}

\begin{lemma}\label{lem:skala}
If $m_k\geq \pi n$, then $m_k^{11} \geq \pi n \gamma_\text{train}$, and $m_k^{22}\geq \pi n (1-\gamma_\text{train})$, with high probability.  If $\max_{k,l} \frac{m_k}{m_l}\leq \delta$, then $\max_{k,l} \frac{m^{11}_k}{m^{11}_l} \leq \delta + o(1)$ with high probability.
\end{lemma}
\begin{proof}
The result follows from \cite{skala2013hypergeometric}.
\end{proof}

\noindent{\bf Proof for Theorem~\ref{cor:nok}}
\begin{proof}
First we note that by a similar argument as in Lemma~\ref{prop:sdp_exactrecovery}, $|m_k^{22}-(1-\gamma_{\text{train}})m_k| \leq c\sqrt{m_k \log n}$ for all $k$ with probability at least $1-rn^{-3}$. Then the size of the smallest cluster of the test graph $A^{22}$ will be of the same order as $n\pi_{\text{min}}$. Also $A^{22}$ has size $\Theta(n)$. $A^{22}$ is independent of any $\hat{X}^{22}$. Thus in Theorem~\ref{thm:matrcv_general}, applying Lemma~\ref{prop:underestimate} and Lemma~\ref{prop:overestimate} to $A^{22}$ shows
$$
\epsilon_{under}= \Omega(n\rho \pi_{\text{min}}^2 /r^2),
$$
$$
\epsilon_{over}= (1+B_{\max})\sqrt{r_T\log n} + B_{\max}r,
$$
and Lemma~\ref{prop:sdp_exactrecovery} shows $\epsilon_{est}=0$, w.h.p. For fixed $r, \pi_{\min}$, we have $\epsilon_{under} \gg \epsilon_{over}$.
By Theorem~\ref{thm:matrcv_general}, choosing
$\Delta=(1+B_{\max}) \sqrt{r_T\log n}+B_{\max}r$ leads to MATR-CV returning the correct $r$. We can further refine $\Delta$ by noting that $
    r_{\max}:=\arg\max_{r_t}\langle A, \hat{X}_{r_t} \rangle \geq r
    $ w.h.p., then it suffices to consider the candidate range $\{r_1, \dots, r_{\max}\}$. The same arguments still hold for this range, thus $r_T$ and $r$ in $\Delta$ can be replaced with $r_{\max}$.
\end{proof}

\subsection{Proof of Theorem~\ref{thm:nok_mmsb}}
\label{subsec:proofmmsb}
 
In the following, we show theoretical guarantees of using MATR-CV to do model selection on MMSB with the SPACL algorithm proposed by \cite{Mao2017EstimatingMM}. We assume $A$ has self-loops for clarify of exposition. Adding the diagonal terms introduces a term that is asymptotically negligible compared with other terms, thus does not change our results.

First we have the following concentration lemma regarding the criterion $\langle A^2-\text{diag}(A^2), X \rangle$ for a general normalized clustering matrix $X$, which will be used to derive the three errors in Theorem~\ref{thm:matrcv_general}.

\begin{lemma}
For any general normalized clustering matrix $X$ satisfying $X\mathbf{1}_n=\mathbf{1}_n$, and an adjacency matrix $A$ generated from its expectation matrix $P$ independent of $X$, w.h.p.
\begin{align}
    \langle \hat{S}-S, X\rangle = O(n\rho\sqrt{\log n}),
\end{align}
where $\hat{S}=A^2-\text{diag}(A^2)$, $S=P^2-\text{diag}(P^2)$.
\label{lem:A2_conc}
\end{lemma}

\begin{proof}

\begin{align}
\langle \hat{S}-S, X\rangle&=\sum_{j, i\neq k}(A_{ij}A_{jk}-P_{ij}P_{jk})X_{ik}\notag\\
&=\underbrace{\sum_{i,j,k}(A_{ij}A_{jk}-E[A_{ij}A_{jk}])X_{ik}}_{\text{Part (i)}}-\underbrace{\sum_{i,j}(A_{ij}-E[A_{ij}])X_{ij}}_{\text{Part (ii)}}
\label{eq:part1part2_A2}
\end{align}

To bound Part (i), we will first bound $f(A_{ij},1\leq i\le j \leq n)=\sum_{i,j,k}A_{ij}A_{jk}X_{ik}/2$.  Let
\begin{align*}
f_{uv}:=f(A_{ij},1\leq i\le j \leq n, A_{uv}=0)=f(A)-\frac{A_{uv}\sum_{k}A_{vk}X_{uk}+A_{uv}\sum_{i}A_{iu}X_{iv}}{2},
\end{align*}
Clearly, $0\leq f-f_{uv}\leq 1$ since $X\mathbf{1}_n=\mathbf{1}_n$, and 
\begin{align*}
\sum_{u<v} (f-f_{uv})&\leq \sum_{u<v}\frac{A_{uv}\sum_{k}A_{vk}X_{uk}+A_{uv}\sum_{i}A_{iu}X_{iv}}{2}\\
&\leq \frac{1}{2}\sum_{u,v}\frac{A_{uv}\sum_{k}A_{vk}X_{uk}+A_{uv}\sum_{i}A_{iu}X_{iv}}{2}\\
&\leq \frac{1}{2}\sum_{u,v,k}\frac{A_{uv}A_{vk}X_{uk}+A_{iu}A_{uv}X_{iv}}{2}\\
&\leq f/2
\end{align*}
So $f$ is $(1/2,0)$ self bounding. Hence,
\begin{align*}
P(f-E[f]\geq t)\leq \exp\left(\frac{-t^2}{E[f]+t}\right)\\
P(f-E[f]\leq -t)\leq \exp\left(\frac{-t^2}{E[f]+2t/3}\right)
\end{align*}
Note that $E[f]=\Theta(n^2\rho^2)$, and so setting $t=\Theta(n\rho\sqrt{\log n})$ we get:
\begin{align*}
P(f-E[f]\geq t)= O(1/n)\\
P(f-E[f]\leq -t)=O(1/n)
\end{align*}

Using the same argument on Part (ii), shows that it is $O(\sqrt{n\rho\log n})$ w.h.p.

Therefore using Eq~\eqref{eq:part1part2_A2}, we see that:
\begin{align*}
\langle \hat{S}-S,X\rangle =O(n\rho\sqrt{\log n})
\end{align*}
w.h.p.
\end{proof}

The next two propositions are for general underfitting and overfitting normalized clustering matrices $\hat{X}$ independent of $A$, which will be used to derive $\epsilon_{\text{under}}$ and $\epsilon_{\text{over}}$.


\begin{Proposition}\label{prop:underestimatemmsb}
For a sequence of underfitting normalized clustering matrix $\{\hat{X}_{r_t}\}_{r_t<r}$, all independent of $A$, then with high probability, $\max_{r_t<r}\langle A^2-\text{diag}(A^2), \hat{X}_{r_t} \rangle \leq \langle A^2-\text{diag}(A^2), X_0 \rangle - \Omega(n^2\rho^2)$.
\end{Proposition}

\begin{proof}
Suppose $P=\Theta B \Theta^T$ has eigenvalue decomposition $Q\Lambda Q^T$. Also consider the singular value decomposition of $\hat {\Theta} = \hat{U} \hat{D} \hat{V}^T$, then $\hat{X}_{r_t} = \hat{\Theta} (\hat{\Theta} ^T \hat{\Theta})^{-1} \hat{\Theta}^T = \hat{U}\hat{U}^T$.  We have
\begin{align}
\langle P^2, X_0 - \hat{X}_{r_t} \rangle & = \tr{P^2} - \tr{P^2\hat{U}\hat{U}^T}    \notag\\
& = \tr{\Lambda^2} - \tr{\Lambda^2Q^T \hat{U}\hat{U}^T Q}  \notag\\
& = \tr{\Lambda^2} - \tr{\Lambda^2\hat{Q}^T \hat{Q}} \notag\\
& \geq \min_{i}\Lambda_{i,i}^2 =  (\lambda^*(P))^2,  
\label{eq:mmsb_under_pop}
\end{align}
where $\hat{Q}:= \hat{U}^T Q$, $\lambda^*(P)$ is the smallest singular value of $P$. The last line follows from Von Neumann's trace inequality and the fact that $\norm{\hat{Q}^T\hat{Q}}_{op} \leq 1$ and $\text{rank}(\hat{Q}^T\hat{Q})\leq r_t<r$.

Applying Lemma B.4 and Lemma 3.6 in \cite{Mao2017EstimatingMM} to \eqref{eq:mmsb_under_pop}, with probability at least $1-r_0\exp\left(-\frac{n}{36\nu^2(1+\alpha_0)^2}\right)$, 
\begin{align}
    \langle P^2, X_0 - \hat{X}_{r_t} \rangle \geq \frac{n^2(\lambda^*(B))^2}{4\nu^2(1+\alpha_0)^2},
\end{align}
that is, $\langle P^2, X_0 - \hat{X}_{r_t} \rangle = \Omega_P(n^2\rho^2)$.

Now 
\begin{align}
    |\langle \text{diag}(P^2), X_0 - \hat{X}_{r_t} \rangle | \leq 2\max_i(P^2)_{ii}\tr{X_0} = O(n\rho^2).
    \label{eq:diagP2}
\end{align}
By Lemma~\ref{lem:A2_conc}, with a union bound since $r$ is fixed,
$$
\max_{r_t<r} |\langle \hat{S} -S, \hat{X}_{r_t} \rangle| = O_P(n\rho\sqrt{\log n})
$$
We have the desired inequality. 


\end{proof}

\begin{Proposition}\label{prop:overestimatemmsb}
For a sequence of overfitting normalized clustering matrix $\{\hat{X}_{r_t}\}_{r<r_t\leq r_T}$ independent of $A$, with high probability, $$\max_{r<r_t\leq r_T}\langle A^2-\text{diag}(A^2), \hat{X}_{r_t} \rangle \leq \langle A^2-\text{diag}(A^2), X_0 \rangle + O(n\rho \sqrt{\log n})$$.
\end{Proposition}

\begin{proof}
First note by a similar argument as~\eqref{eq:mmsb_under_pop}, 
\begin{align}
   \langle P^2, X_0 - \hat{X}_{r_t} \rangle = \tr{\Lambda^2} - \tr{\Lambda^2 \hat{Q}^T\hat{Q}} \geq 0.
\end{align}
Since $\hat{X}_{r_t}$ and $A$ are independent, by an argument similar to~\eqref{eq:diagP2} and Lemma~\ref{lem:A2_conc}, 
\begin{align}
   \max_{r<r_t\leq r_T}\langle A^2-\text{diag}(A^2), \hat{X}_{r_t} \rangle \leq \langle A^2-\text{diag}(A^2), X_0 \rangle + O(n\rho\sqrt{\log n})
\end{align}
w.h.p.
\end{proof}

Next we derive the estimation error by considering the parameter estimates recovered by SPACL \citep{Mao2017EstimatingMM}. For notational convenience, we first derive some bounds for $\hat{\Theta}$ and $\hat{B}$ estimated from the matrix $A$, with the understanding that the same asymptotic bounds apply to estimates obtained from the training graph provided the split is random and the number of training nodes is $\Theta(n)$. The next lemma states the parameter estimation error from \cite{Mao2017EstimatingMM}. 

\begin{lemma}
There exists a permutation matrix $\Pi$, such that, with probability larger than $1-O(rn^{-2})$,
\begin{align}
    \Delta_1 & := \norm{\hat{\Theta}-\Theta\Pi}_F = O\left(\frac{(\log n)^{1+\xi}}{\sqrt{\rho}}\right),  \label{eq:thetabound}  \\
    \Delta_2 &:= \norm{\hat{B}- \Pi ^T B \Pi}_F = O(\sqrt{\frac{\rho}{n}} (\log n)^{1+\xi}).
\end{align}

\label{lem:mmsb_para_error} 
\end{lemma}

\begin{proof}
These bounds follow directly from Corollary 3.7 of \cite{Mao2017EstimatingMM}, with $\alphav, r, \lambda^*(B/\rho)$ all being constant.
\end{proof}


In what follows, we omit the permutation matrix $\Pi$ to simplify notation. If $\Pi$ is not the identity matrix, we can always redefine $\Theta$ as $\Theta \Pi$, and $B$ as $\Pi^TB \Pi$. This would not affect the results, since we want to prove bounds on normalized clustering matrices where $\Pi$ always cancels out, i.e., $X = \Theta \Pi ((\Theta \Pi)^T \Theta \Pi)^{-1}(\Theta \Pi)^T = \Theta (\Theta^T \Theta)^{-1}\Theta^T.$

We are interested in bounding the estimation error in $\hat{H}=\hat{B}^{-1}(\hat{\Theta}^T\hat{\Theta})^{-1}\hat{\Theta}^T$. In order to build up the bound, we will make repeated use of the following two facts.

\begin{fact}\label{lem:prodminus}
For general matrices $C, \hat{C}, D, \hat{D}$,
$$\norm{\hat{C}\hat{D} - CD}_F \leq \norm{(\hat{C}-C)(\hat{D}-D)}_F + \norm{(\hat{C}-C)D}_F + \norm{C(\hat{D}-D)}_F$$
\end{fact}

\begin{proof}
The proof follows directly from expansion and the triangle inequality,
$$\hat{C}\hat{D} - CD = (\hat{C}-C)(\hat{D}-D) + (\hat{C}-C)D + C(\hat{D}-D).$$

\end{proof}

\begin{fact}
\label{lem:invminus}
For a general matrices $C$, $\hat{C}$, assume $\norm{(C-\hat{C})C^{-1}}_F< 1$, then 
$$\norm{\hat{C}^{-1} - C^{-1}}_F \leq \frac{\norm{C^{-1}}_F \norm{(C-\hat{C}) C^{-1}}_F}{1- \norm{(C-\hat{C})C^{-1}}_F}$$
\end{fact}
\begin{proof}
First decompose
\begin{align}
    \hat{C}^{-1}-C^{-1} = \hat{C}^{-1}CC^{-1} - C^{-1}=(\hat{C}^{-1}C-I)C^{-1}=\hat{C}^{-1}(C-\hat{C}) C^{-1}.
    \label{eq:invminus_decomp}
\end{align}
Taking Frobenius norms,
$$\norm{\hat{C}^{-1}}_F \leq \norm{C^{-1}}_F + \norm{\hat{C}^{-1} (C-\hat{C})C^{-1}}_F \leq \norm{C^{-1}}_F + \norm{\hat{C}^{-1}}_F \norm{(C-\hat{C})C^{-1}}_F .$$
Rearranging, 
$$\norm{\hat{C}^{-1}}_F\leq \frac{\norm{C^{-1}}_F}{1- \norm{(C-\hat{C}) C^{-1}}_F}.$$
Applying this to~\eqref{eq:invminus_decomp},
$$\norm{\hat{C}^{-1} -C^{-1}}_F\leq \norm{\hat{C}^{-1}}_F \norm{(C-\hat{C}) C^{-1}}_F \leq \frac{\norm{C^{-1}}_F \norm{(C-\hat{C}) C^{-1}}_F}{1- \norm{(C-\hat{C}) C^{-1}}_F}.$$
\end{proof}

Next we have a lemma bounding the error in estimating the quantity $H = B^{-1}(\Theta^T\Theta)^{-1}\Theta^T$.

\begin{lemma}
Let $H = B^{-1}(\Theta^T\Theta)^{-1}\Theta^T$, $\hat{H} = \hat{B}^{-1}(\hat{\Theta}^T\hat{\Theta})^{-1}\hat{\Theta}^T$, then w.h.p.
\begin{align}
    \norm{H-\hat{H}}_F = O\left( \frac{(\log n)^{1+\xi}}{n\rho^{3/2}}\right)   \notag\\
    \norm{H}_F=O\left(\frac{1}{\sqrt{n}\rho}. \right)
\end{align}

\label{lem:Hhat}
\end{lemma}

\begin{proof}
We build up the estimator of $H$ step by step by repeatedly using Facts~\ref{lem:prodminus} and \ref{lem:invminus}. Denote $F_1=\norm{(\Theta^T\Theta)^{-1}}_F$, and $F_2 = \norm{B^{-1}}_F$, by Lemma 3.6 in \cite{Mao2017EstimatingMM}, 
\begin{align}
 F_1 \leq \sqrt{r_0} \norm{(\Theta^T\Theta)^{-1}}_{op} = O_P(1/n),    \label{eq:fnorm_F1}
\end{align}
and 
\begin{align}
    F_2 \leq \sqrt{r}\norm{B^{-1}}_{op} = O(1/\rho).
    \label{eq:fnorm_F2}
\end{align}

First, applying Fact~\ref{lem:prodminus}, 
\begin{align*}
    \norm{\hat{\Theta}^T\hat{\Theta} - \Theta ^T \Theta}_F &\leq \Delta_1^2 + 2\norm{\Theta}_{F} \Delta_1    \\
    \norm{(\Theta ^T \Theta)^{-1}}_F \norm{\hat{\Theta}^T\hat{\Theta} - \Theta ^T \Theta}_F & \leq (\Delta_1^2 + 2\norm{\Theta}_{F} \Delta_1)F_1    \\
    & = O_P\left(\frac{(\log n)^{2+2\xi}}{n\rho}\right) + O_P\left( \frac{(\log n)^{1+\xi}}{\sqrt{n\rho}}   \right) \\
    & = O_P\left( \frac{(\log n)^{1+\xi}}{\sqrt{n\rho}}   \right),
\end{align*}
using Lemma \ref{lem:mmsb_para_error} and eq~\eqref{eq:fnorm_F1}. Thus for large $n$, $\norm{(\Theta ^T \Theta)^{-1}}_F \norm{\hat{\Theta}^T\hat{\Theta} - \Theta ^T \Theta}_F<1/2$.  Then using Fact~\ref{lem:invminus}, we have 
\begin{align}
    \norm{(\hat{\Theta}^T\hat{\Theta})^{-1} - (\Theta^T\Theta)^{-1}}_F &\leq \frac{\norm{(\Theta^T\Theta)^{-1}}_F \norm{((\Theta^T\Theta)-(\hat{\Theta}^T\hat{\Theta})) (\Theta^T\Theta)^{-1}}_F}{1- \norm{((\Theta^T\Theta)-(\hat{\Theta}^T\hat{\Theta}))(\Theta^T\Theta)^{-1}}_F} \notag\\
    & \leq \frac{\norm{(\Theta^T\Theta)^{-1}}_F^2 \norm{((\Theta^T\Theta)-(\hat{\Theta}^T\hat{\Theta}))}_F}{1- \norm{((\Theta^T\Theta)-(\hat{\Theta}^T\hat{\Theta}))(\Theta^T\Theta)^{-1}}_F} \notag\\
    & \leq 2\norm{(\Theta^T\Theta)^{-1}}_F^2 \norm{((\Theta^T\Theta)-(\hat{\Theta}^T\hat{\Theta}))}_F \notag\\
    & \leq (\Delta_1^2 + 2\norm{\Theta}_{F} \Delta_1)F_1^2 = O_P\left(\frac{(\log n )^{1+\xi}}{n^{3/2}\rho^{1/2}}.
    \label{eq:thetainv_error}
    \right)
\end{align} 

Similarly using Lemma~\ref{lem:mmsb_para_error} and eq~\eqref{eq:fnorm_F2}, by noting that $$\norm{B^{-1}}_F\norm{(B-\hat{B})}_F = \Delta_2 F_2 = O(\frac{(\log n)^{1+\xi}}{\sqrt{n\rho}})< 1/2$$ for large $n$ w.h.p.,
\begin{align}
    \norm{\hat{B}^{-1} - B^{-1}}_F &\leq  \frac{\norm{B^{-1}}_F \norm{(B-\hat{B}) B^{-1}}_F}{1- \norm{(B-\hat{B})B^{-1}}_F}    \notag\\
    & \leq 2 \norm{B^{-1}}_F^2\norm{(B-\hat{B})}_F \notag\\
    & \leq 2\Delta_2 F_2^2 = O_P\left( \frac{(\log n)^{1+\xi}}{n^{1/2}\rho^{3/2}}\right)
    \label{eq:Binv_error}
\end{align}
using Fact~\ref{lem:invminus}.

Next applying Fact~\ref{lem:prodminus} to $G := B^{-1}(\Theta^T\Theta)^{-1}$ and its estimate $\hat{G} := \hat{B}^{-1}(\hat{\Theta}^T\hat{\Theta})^{-1}$,

\begin{align*}
& \norm{G - \hat{G}}_F  \notag\\
  \leq & \norm{\hat{B}^{-1} - B^{-1}}_F F_1 +
 (\norm{B^{-1}}_F+ \norm{B^{-1}- \hat{B}^{-1}}_F)\norm{(\hat{\Theta}^T\hat{\Theta})^{-1} - (\Theta^T\Theta)^{-1}}_F \\
 \leq & 2\Delta_2 F_2^2F_1+ (2\Delta_2 F_2^2 + F_2)(\Delta_1^2 + 2\norm{\Theta}_{F} \Delta_1)F_1^2 \notag\\
= & O_P\left( \frac{(\log n)^{1+\xi}}{(n\rho)^{3/2}}\right), 
\end{align*}
 using Eqs~\eqref{eq:fnorm_F1}-\eqref{eq:Binv_error}.

Finally, since $H = G\Theta^T$, and $\hat{H} = \hat{G}\hat{\Theta}^T$, 
\begin{align*}
    \norm{H-\hat{H}}_F & \leq \norm{\hat{G} - G}_F \norm{\Theta}_{F} + (\norm{G}_F+ \norm{G- \hat{G}}_F)\norm{\hat{\Theta}-\Theta}_F \\
    & \leq \sqrt{n} \norm{G - \hat{G}}_F + (F_1F_2+ \norm{G - \hat{G}}_F) \Delta_1    \\
    & = O_P\left( \frac{(\log n)^{1+\xi}}{n\rho^{3/2}}\right),
\end{align*}
and 
\begin{align*}
    \norm{G}_F^2 & = \text{tr} ((\Theta^T\Theta)^{-1}(\Theta^T\Theta)^{-1} (BB^T)^{-1})\\
    & \leq \norm{(\Theta^T\Theta)^{-1}}_{op}^2 \text{tr}((BB^T)^{-1})\\
    & = O_P(1/n^2) F_2^2
\end{align*}
 by Eq~\eqref{eq:fnorm_F1},
$\norm{H}_F=O_P(\frac{1}{\sqrt{n}\rho})$ follows.
\end{proof}

\begin{lemma}Consider applying SPACL on the training graph $A^{11}$ to obtain $(\hat{\Theta}^{11})^T$ and $\hat{B}$, and use regression in MATR-CV to estimate membership matrix, i.e., $(\hat{\Theta}^{22})^T = \hat{B}^{-1}((\hat{\Theta}^{11})^T\hat{\Theta}^{11})^{-1}(\hat{\Theta}^{11})^T A^{12}:= \hat{H}A^{12}$, then $\norm{\hat{\Theta}^{22} - \Theta^{22}}_F = O\left( \frac{(\log n)^{1+\xi}}{\sqrt{\rho}}\right) $ w.h.p.
\label{lem:theta_est_error}
\end{lemma}

\begin{proof}
Since $(\Theta^{22})^T =H\Theta^{11}B(\Theta^{22})^T$, where $H :={B}^{-1}((\Theta^{11})^T{\Theta}^{11})^{-1}(\Theta^{11})^T$,  
\begin{equation*}
\begin{aligned}
    (\hat{\Theta}^{22})^T-(\Theta^{22})^T &= \hat{H}{A}^{12}-H\Theta^{11}B(\Theta^{22})^T\\
    &=\hat{H}{A}^{12}-\hat{H}\Theta^{11}B(\Theta^{22})^T+\hat{H}\Theta^{11}B(\Theta^{22})^T-H\Theta^{11}B(\Theta^{22})^T\\
    &=\hat{H}(A^{12}-\Theta^{11}B(\Theta^{22})^T)+(\hat{H}-H)\Theta^{11}B(\Theta^{22})^T\\
    &= \underbrace{H(A^{12}-P^{12})}_{Q_1}+\underbrace{(\hat{H}-H)(A^{12}-P^{12})}_{Q_2}+\underbrace{(\hat{H}-H)P^{12}}_{Q_3}.
\end{aligned}
\end{equation*}
For $Q_1$, 
\begin{align*}
    \norm{Q_1}_F &\leq \norm{A^{12}-P^{12}}_{op} \norm{H}_F\\  
    &= O_P(\sqrt{n\rho})O_P(\frac{1}{\sqrt{n}\rho}) = O_P(1/\sqrt{\rho})
\end{align*}
by Lemma~\ref{lem:Hhat}. For $Q_2$,
\begin{align*}
    \norm{Q_2}_F &\leq \norm{A^{12}-P^{12}}_{op} \norm{\hat{H}-H}_F\\  
    &= O_P(\sqrt{n\rho})O_P\left( \frac{(\log n)^{1+\xi}}{n\rho^{3/2}}\right) = O_P\left( \frac{(\log n)^{1+\xi}}{\sqrt{n}\rho}\right),
\end{align*}
 by Lemma~\ref{lem:Hhat}. Finally for $Q_3$, 
\begin{align*}
    \norm{Q_3}_F &\leq \norm{P^{12}}_{F} \norm{\hat{H}-H}_F\\  
    &= O_P(n\rho)O_P\left( \frac{(\log n)^{1+\xi}}{n\rho^{3/2}}\right) = O_P\left( \frac{(\log n)^{1+\xi}}{\sqrt{\rho}}\right),
\end{align*}
The above arguments lead to

\begin{align*}
    \norm{\hat{\Theta}^{22} - \Theta^{22}}_F 
    = O_P\left( \frac{(\log n)^{1+\xi}}{\sqrt{\rho}}\right).
\end{align*}


\end{proof}

\begin{Proposition}
\label{prop:estmmsb}
Given the correct number of clusters $r$, then with high probability, $\langle A^{22},\hat{X}_{r}^{22} \rangle > \langle A^{22},X_0^{22} \rangle -O((n\rho)^{3/2}(\log n)^{1+\xi}).$

\end{Proposition}

\begin{proof}
First note
\begin{align}
    |\langle (A^{22})^2-\text{diag}((A^{22})^2),\hat{X}_{r}^{22}- {X}_0^{22}\rangle| & \leq |\langle \hat{S}^{22}-S^{22},\hat{X}_{r}^{22}- {X}_0^{22}\rangle|  \notag\\
    & \quad + |\langle S^{22},\hat{X}_{r}^{22}- {X}_0^{22}\rangle|,
    \label{eq:A2_estX}
\end{align}
where $\hat{S}^{22} = (A^{22})^2-\text{diag}((A^{22})^2)$, $S^{22}=(P^{22})^2-\text{diag}((P^{22})^2)$.

Consider SVD of $\hat{\Theta}^{22} = \hat{U}\hat{D} \hat{V}^T$ and $\Theta^{22} = U D V^T,$ then $\hat{X}_{r}^{22} = \hat{U}\hat{U}^T$, and $X_0^{22} = UU^T.$ For any orthogonal matrix $O$,
\begin{align}
    \norm{\hat{X}_{r}^{22}- {X}_0^{22}}_F &= \norm{\hat{U}\hat{U}^T - UU^T}_F = \norm{\hat{U}\hat{U}^T - UO(UO)^T}_F\notag\\
    &\leq \norm{(\hat{U}-UO)(\hat{U}-UO)^T}_F + 2 \norm{(\hat{U}-UO)(UO)^T}_F\notag\\
    &\leq \norm{\hat{U}-UO}_F^2 + 2 \norm{\hat{U}-UO}_F
    \label{eq:x_error}
\end{align}
Using the Theorem 2 in \cite{davis-kahan}, we know there exists $O$ such that, 

$$\norm{\hat{U}- UO}_F \leq \frac{2\norm{\hat{\Theta}^{22} - \Theta^{22}}_F}{\lambda_{r}(\Theta^{22})},$$
where $\lambda_{r}(\Theta^{22})$ is the $r$-th largest singular value of $\Theta^{22}$. Using Lemma 3.6 in \cite{Mao2017EstimatingMM}, w.h.p, $\lambda_{r_0}(\Theta^{22}) =\Omega(\sqrt{n}).$ Now by Lemma~\ref{lem:theta_est_error} and Eq~\eqref{eq:x_error}, w.h.p.
$$\norm{\hat{X}_{r_0}^{22}- {X}_0^{22}}_F = O(\frac{(\log n)^{1+\xi}}{\sqrt{n\rho}}).$$

Now in Eq~\eqref{eq:A2_estX}, 
$$
|\langle S^{22},\hat{X}_{r}^{22}- {X}_0^{22}\rangle| \leq \norm{S^{22}}_F \norm{\hat{X}_{r}^{22}- {X}_0^{22}}_F = O_P((n\rho)^{3/2}(\log n)^{1+\xi}),
$$
and
$$
|\langle \hat{S}^{22}-S^{22},\hat{X}_{r}^{22}- {X}_0^{22}\rangle| = O_P(n\rho\sqrt{log n})
$$
by Lemma~\ref{lem:A2_conc}.
\end{proof}

Finally we prove Theorem~\ref{thm:nok_mmsb}.

\begin{proof}[Proof of Theorem~\ref{thm:nok_mmsb}]
By Propositions~\ref{prop:underestimatemmsb}, \ref{prop:overestimatemmsb} and \ref{prop:estmmsb}, 
\begin{align*}
   \epsilon_{\text{under}} & = \Omega(n^2\rho^2),  \\ 
   \epsilon_{\text{est}} & = O((n\rho)^{3/2}(\log n)^{1+\xi}),    \\
   \epsilon_{\text{over}} &= O(n\rho\sqrt{\log n}).
\end{align*}
Then the result follows by setting $\Delta=O((n\rho)^{3/2}(\log n)^{1+\xi})$.
\end{proof}





\section{Detailed parameter settings in experiments and additional results}
\label{sec:expdetail}
\subsection*{Motivating examples in Section~\ref{sec:knownk} (Figure~\ref{fig:yd_motivation})} 

In Figure~\ref{fig:yd_motivation}(a), we generate an adjacency matrix from a SBM model with four communities, each having 50 nodes, and
$$B = \begin{bmatrix} 
0.8 & 0.6 & 0.4 & 0.4 \\
0.6 & 0.8 & 0.4 & 0.4 \\
0.4 & 0.4 & 0.8 & 0.6 \\
0.4 & 0.4 & 0.6 & 0.8 
\end{bmatrix}.$$
The visualization of the underlying probability matrix is shown in Figure~\ref{fig:yd_motivation_visual}(a).

In Figure~\ref{fig:yd_motivation}(b), we consider a four-component Gaussian mixture model, where the means $\mu_1, \dots, \mu_4$ are generated from Gaussian distributions centered at $(0,0), (0,0), (5,5), (10,10)$ with covariance $6 I$, so that the first two clusters are closer to each other than the rest. Then we generate 1000 data points centered at these means with covariance $0.5 I$,  each point assigned to one of the four clusters independently with probability $(\frac{20}{42}, \frac{20}{42}, \frac{1}{42}, \frac{1}{42})$. Finally, we introduce correlation between the two dimensions by multiplying each point by $\begin{bmatrix} 
2 & 1 \\
1 & 2 
\end{bmatrix}$. A scatter plot example of the datapoints is shown in Figure~\ref{fig:yd_motivation_visual}(b).

 \begin{figure}[htp!]
	 \begin{subfigure}{.5\textwidth}
	 		\centering
	 		\includegraphics[width=0.8\linewidth]{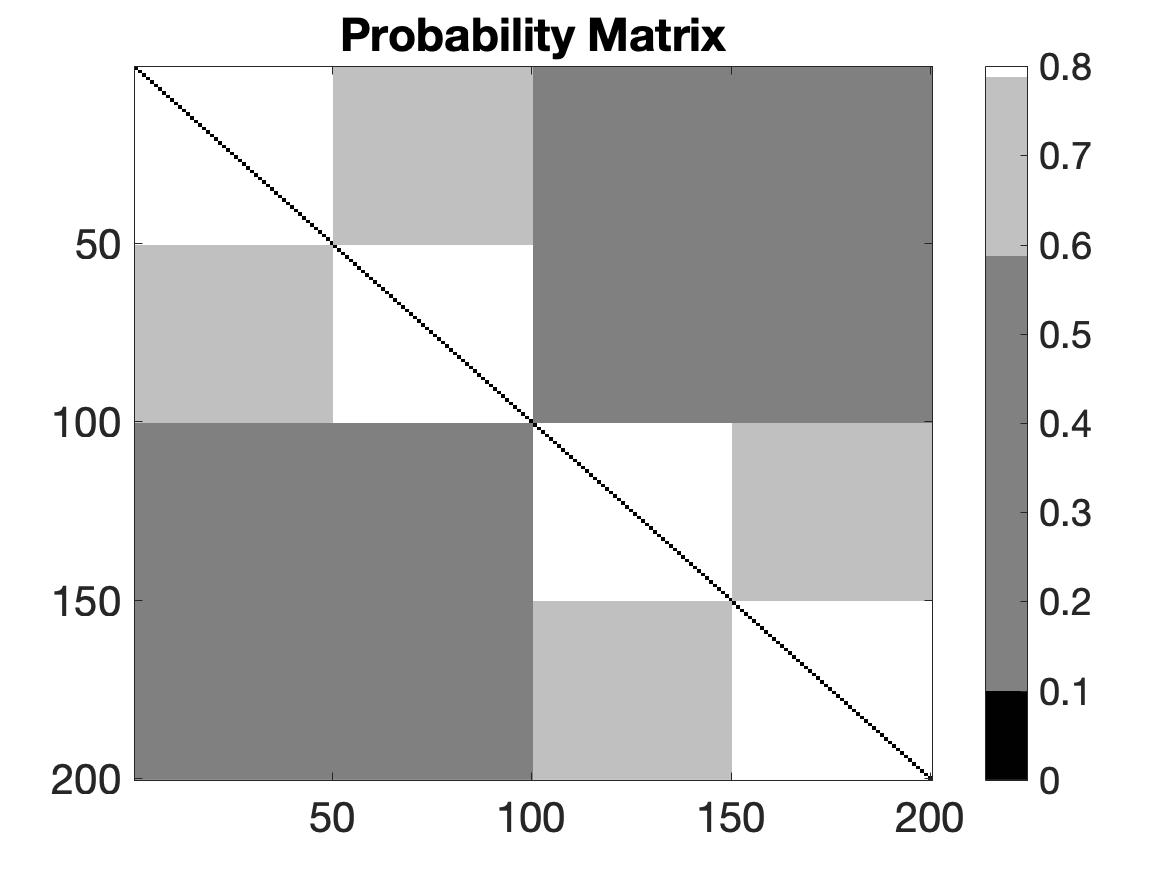} 
	 		\label{sfig_motivation3}
	 		\caption{SBM}
	 	\end{subfigure}
	 	\begin{subfigure}{.5\textwidth}
	 		\centering
	 		\includegraphics[width=0.8\linewidth]{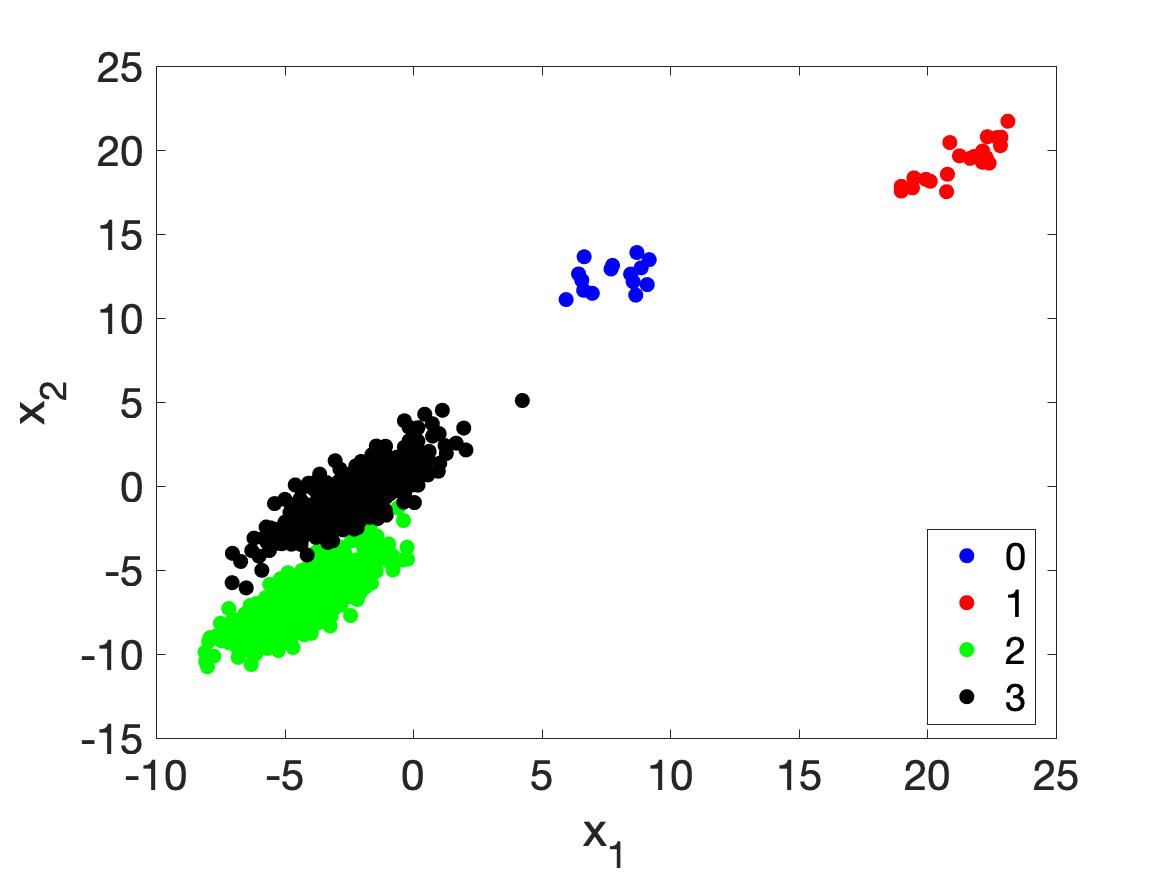}
	 		\label{fig:sfig_motivation1}
	 		\caption{Gaussian mixture}
	 	\end{subfigure}%
	\caption{Datasets used for Figure~\ref{fig:yd_motivation}.}
	\label{fig:yd_motivation_visual}
\end{figure}

\subsection*{Tuning with \ref{SDP:YD} (Figure~\ref{fig:tuning} (a)-(b))}
\textbf{Figure~\ref{fig:tuning} (a):} We consider graphs generated from a hierarchical SBM  with equal sized clusters, where 
$$B = \rho\times \begin{bmatrix} 
0.8 & 0.6 & 0.3 & 0.3 \\
0.6 & 0.8 & 0.3 & 0.3 \\
0.3 & 0.3 & 0.8 & 0.6 \\
0.3 & 0.3 & 0.6 & 0.8 
\end{bmatrix}.$$
Each cluster has 100 nodes and $\rho$ ranges from $0.2$ to $1$.

\textbf{Figure~\ref{fig:tuning} (b):} Next, we consider graphs generated from a SBM with the same $B$ matrix, but with unequal cluster sizes. 
Cluster 1 and 3 have 100 nodes each, while cluster 2 and 4 have 50 nodes each. $\rho$ ranges from $0.2$ to $1$.

\subsection*{Tuning with spectral clustering (Figure~\ref{fig:tuning}(c)-(d))}
 
We generate the means $\mu_a, a\in[3]$ from $d=20$ dimensional Gaussian distribution with covariance $0.01I$. To impose sparsity on each $\mu_{a}$, we set all but the first two dimensions to 0.  To change the level of clustering difficulty, we multiply $\mu_a$ with a separation constant $c$, and a larger $c$ leads to larger separation and easier clustering. We vary $c$ from $0$ to $200$. We generate $n=500$ samples from each mixture with a constant covariance matrix (an identity matrix) using Eq~\ref{eq:mog}.  For Figure~\ref{fig:tuning} (c), the probabilities of cluster assignment are equal, while for  Figure~\ref{fig:tuning} (d), each point belongs to one of the three clusters with probability $(\frac{20}{22}, \frac{1}{22}, \frac{1}{22})$. 2D projections of the datapoints for the two settings are shown in Figure~\ref{fig:sp2_visual}.

\begin{figure}[htp!]
	\begin{subfigure}[t]{.5\textwidth}
	  \centering
	  \includegraphics[width=0.8\linewidth]{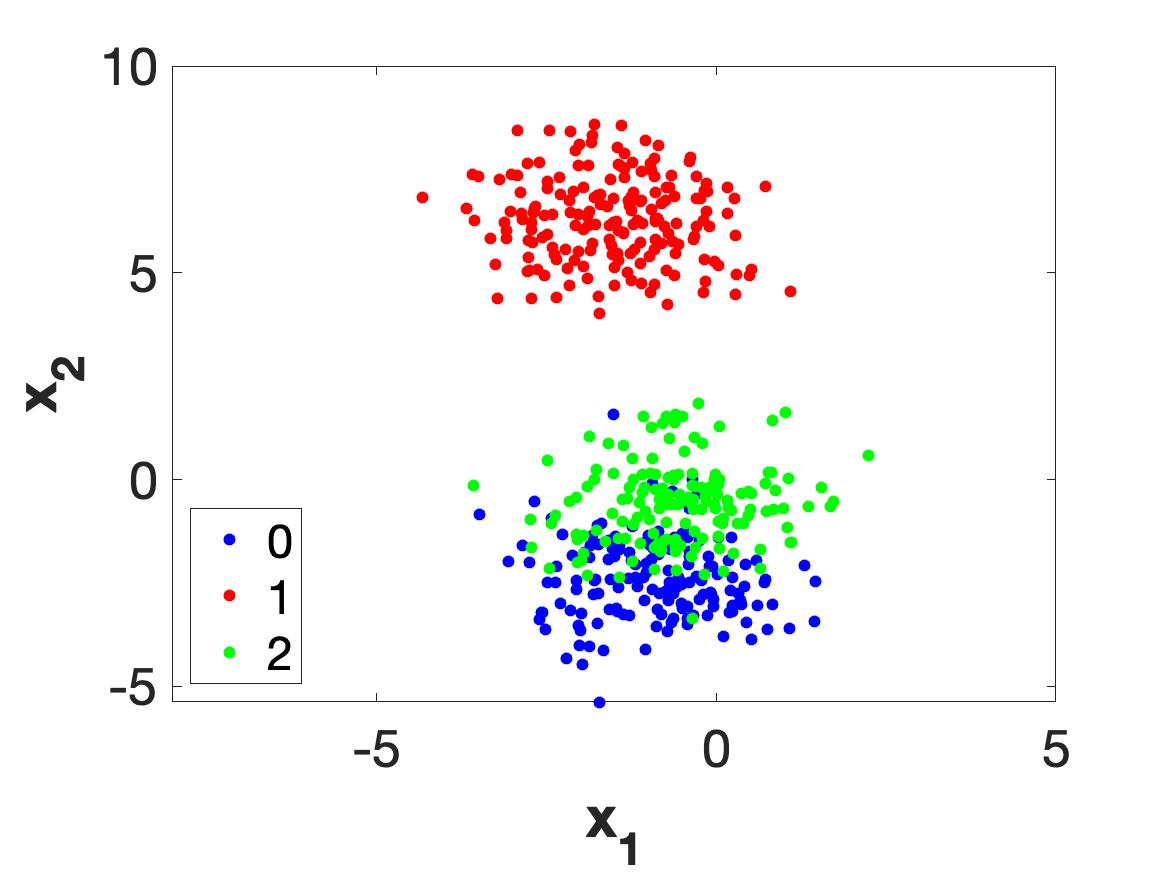}
	  \caption{Equal sized clusters}
	\end{subfigure}\hspace{-3mm}
	\begin{subfigure}[t]{.5\textwidth}
	  \centering
	  \includegraphics[width=0.8\linewidth]{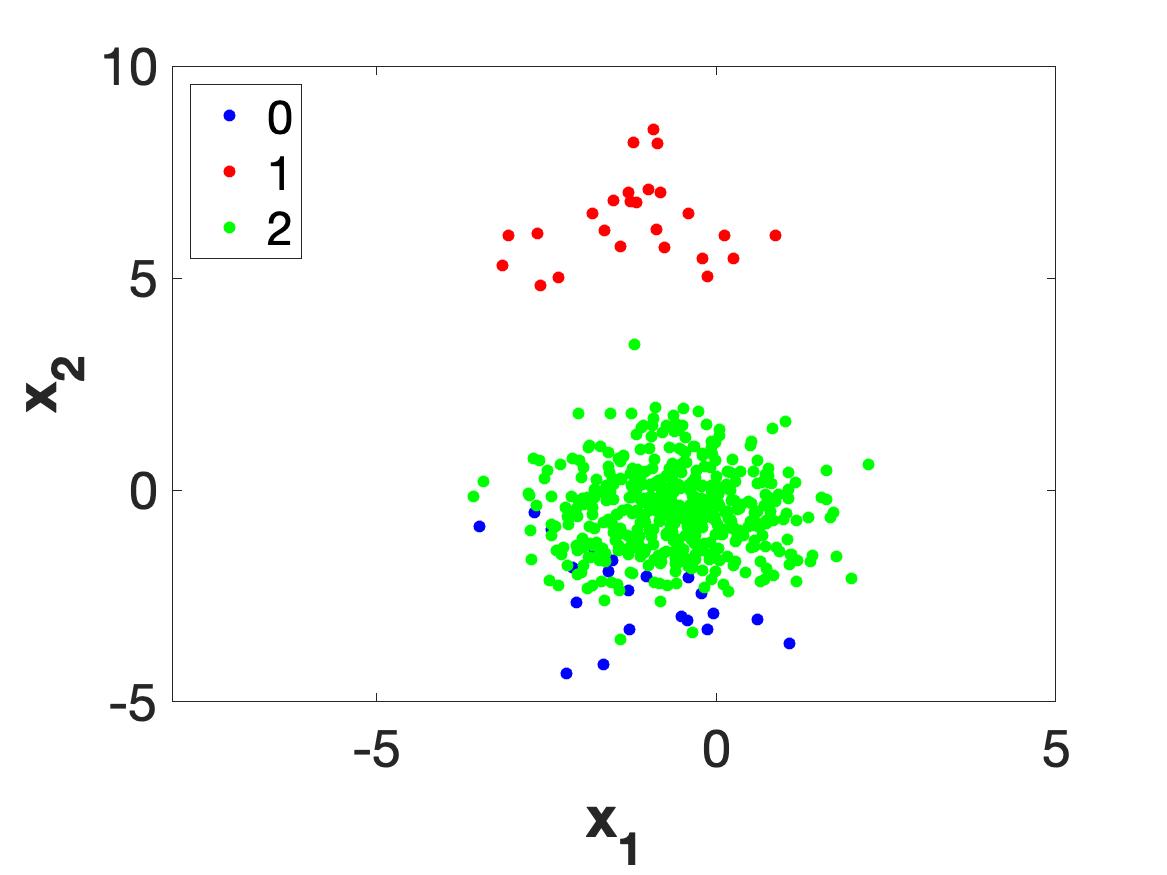}
	  \caption{Unequal sized clusters}
	\end{subfigure}\hspace{-3mm}
	\caption{2D projections of the datapoints for Gaussian mixtures.}
	\label{fig:sp2_visual}\vspace{-3mm}
\end{figure}

\subsection*{Additional figure for Section~\ref{sec:exp_matr_mixture}}

\begin{figure}[h!]
\begin{subfigure}[t]{.33\textwidth}
  \centering
  \includegraphics[width=1\linewidth]{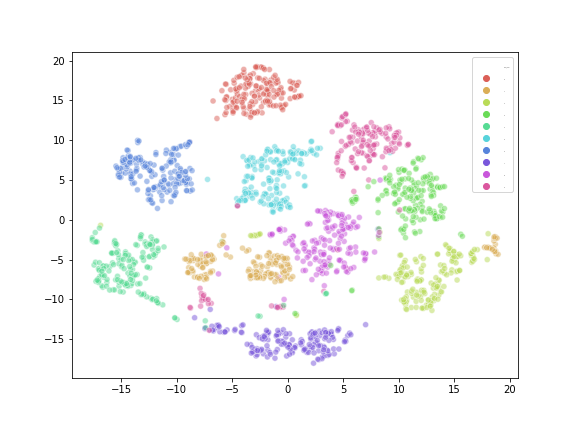}
  \vspace{-5mm}
  \caption{True clustering}
\end{subfigure}\hspace{-3mm}
\begin{subfigure}[t]{.33\textwidth}
  \centering
  \includegraphics[width=1\linewidth]{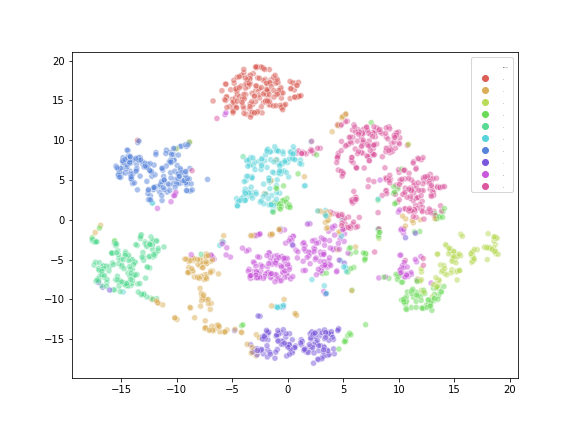}
  \vspace{-5mm}
  \caption{Clustering by MATR}
\end{subfigure}\hspace{-3mm}
\begin{subfigure}[t]{.33\textwidth}
  \centering
  \includegraphics[width=1\linewidth]{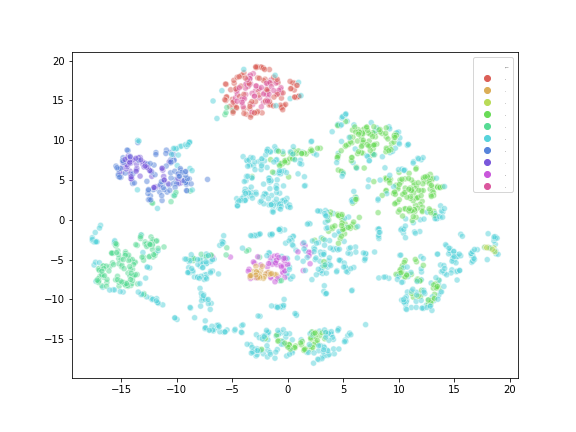} 
  \vspace{-5mm}
  \caption{Clustering by DS}
\end{subfigure}\hspace{-3mm}\\
\centering
\begin{subfigure}[t]{.33\textwidth}
  \centering
  \includegraphics[width=1\linewidth]{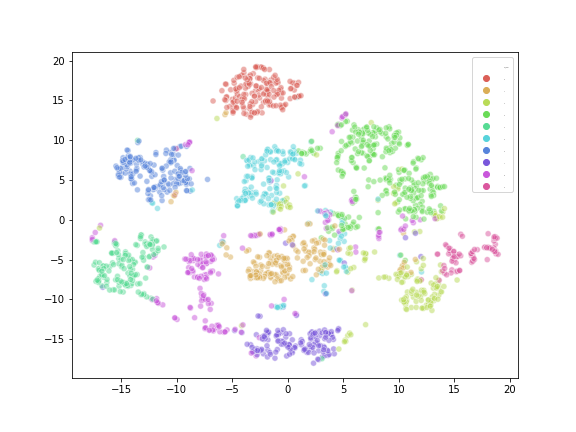}
  \vspace{-5mm}
  \caption{Clustering by KNN}
\end{subfigure}\hspace{-3mm}
\begin{subfigure}[t]{.33\textwidth}
  \centering
  \includegraphics[width=1\linewidth]{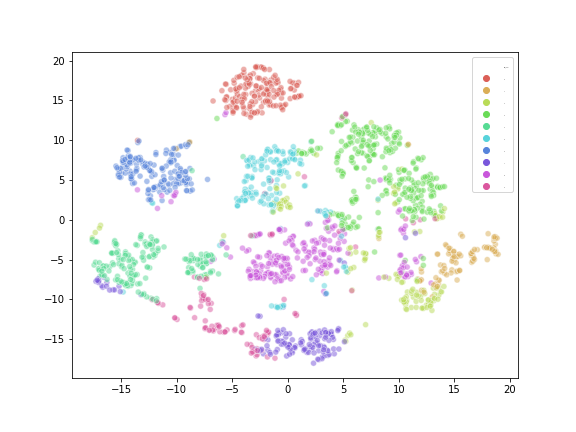}
  \vspace{-5mm}
   \caption{Clustering by MST}
\end{subfigure}%
\caption{Visualization of clustering results on handwritten digits dataset.}
\label{fig:tsne}\vspace{-3mm}
\end{figure}

\subsection*{Tuning with \ref{SDP:BW} (Figure~\ref{fig:matr-cv})}
\textbf{Figure~\ref{fig:matr-cv} (a):} We first consider graphs generated from a SBM with equal sized clusters, where 
$$B = \rho\times \begin{bmatrix} 
0.8 & 0.5 & 0.3 & 0.3 \\
0.5 & 0.8 & 0.3 & 0.3 \\
0.3 & 0.3 & 0.8 & 0.5 \\
0.3 & 0.3 & 0.5 & 0.8 
\end{bmatrix}.$$
Each cluster has $100$ nodes and $5$ $\rho$'s are selected from $0.2$ to $0.6$ with even spacing.

\textbf{Figure~\ref{fig:matr-cv} (b):} Here we consider graphs generated from an unequal-sized SBM , where the $B$ matrix is the same as above. The clusters have $120,80,120,80$ nodes respectively. The same $\rho$'s as above are used.

\textbf{Table~\ref{tab:MATR_CV_K} (a,b):} We show the median number of clusters selected by each method as $\rho$ changes. The ground truth is 4 clusters.


\begin{table}[]
\begin{subtable}[t]{0.50\textwidth}
\centering
\begin{tabular}[t]{|l|l|l|l|}
\hline
  $\rho$   & MATR-CV & BH   & ECV \\
\hline       
0.2  & 2     & 2  & 2 \\
\hline
0.3  & 2     & 2  & 2\\
\hline
0.4  & 4    & 2  & 2 \\
\hline
0.5  & 4   & 2  & 2 \\
\hline
0.6 & 4   & 4 & 2 \\
\hline
\end{tabular}
\caption{Median number of clusters selected for equal size case }
\end{subtable}%
\begin{subtable}[t]{0.50\textwidth}
\centering
\begin{tabular}[t]{|l|l|l|l|}
\hline
  $\rho$    & MATR-CV & BH   & ECV \\
\hline       
0.2  & 2     & 2  & 2 \\
\hline
0.3  & 2     & 2  & 2\\
\hline
0.4  & 3    & 2  & 2 \\
\hline
0.5  & 4    & 2 & 2 \\
\hline
0.6 & 4   & 3 & 2 \\
\hline
\end{tabular}
\caption{Median number of clusters selected for unequal size case }
\end{subtable}%
\caption{Comparison of model selection results along with $\rho$ for all algorithms.}
\label{tab:MATR_CV_K}
\end{table}

\bibliographystyle{authordate1}
\bibliography{reference.bib}
\end{document}